\documentclass[a4paper,11pt]{article}
\usepackage[utf8]{inputenc} 
\usepackage[T1]{fontenc}    
\usepackage[margin=1in]{geometry}
\usepackage{wrapfig}
\usepackage{authblk}
\usepackage{import}
\usepackage{microtype}
\usepackage{graphicx}
\usepackage{xcolor}
\usepackage{booktabs} 

\usepackage{amsmath,amsthm,amssymb,amsfonts}
\usepackage{algorithm}
\usepackage{algorithmic}

\usepackage{comment}
\usepackage{bm}
\usepackage{pifont}

\theoremstyle{plain}
\newtheorem{theorem}{Theorem}[]

\newtheorem{lemma}[]{Lemma}
\newtheorem{proposition}{Proposition}

\newtheorem{definition}{Definition}[]
\newtheorem{assumption}{Assumption}[]
\newtheorem{fact}{Fact}[]
\newtheorem{remark}{Remark}[]

\def\1{\mathbf{1}}

\def\va{{\mathbf{a}}}
\def\vb{{\mathbf{b}}}

\def\vg{{\mathbf{g}}}

\def\vm{{\mathbf{m}}}

\def\vp{{\mathbf{p}}}

\def\vx{{\mathbf{x}}}
\def\vy{{\mathbf{y}}}
\def\vz{{\mathbf{z}}}

\def\mG{{\mathbf{G}}}

\def\mM{{\mathbf{M}}}

\DeclareMathAlphabet{\mathsfit}{\encodingdefault}{\sfdefault}{m}{sl}
\SetMathAlphabet{\mathsfit}{bold}{\encodingdefault}{\sfdefault}{bx}{n}

\def\gA{{\mathcal{A}}}

\def\gC{{\mathcal{C}}}
\def\gD{{\mathcal{D}}}

\def\gG{{\mathcal{G}}}

\def\gI{{\mathcal{I}}}

\def\gO{{\mathcal{O}}}
\def\gP{{\mathcal{P}}}
\def\gQ{{\mathcal{Q}}}

\def\gX{{\mathcal{X}}}

\def\sB{{\mathbb{B}}}

\def\sG{{\mathbb{G}}}

\def\sI{{\mathbb{I}}}

\def\sR{{\mathbb{R}}}

\newcommand{\E}{\mathbb{E}}

\newcommand{\R}{\mathbb{R}}

\DeclareMathOperator*{\argmin}{arg\,min}

\usepackage[algo2e]{algorithm2e}
\usepackage{hyperref}
\usepackage{comment}
\usepackage{subfig}
\usepackage{enumitem}
\setlist{leftmargin=3mm}
\usepackage{soul}

\usepackage{caption}
\usepackage{multirow}

\begin{document}
\title{\bf {\color{black}Robust Training in High Dimensions via \\ \underline{B}lock Coordinate \underline{G}eometric \underline{M}edian \underline{D}escent} }
\date{}
\author[1]{Anish Acharya\thanks{This work is supported in part by NSF grants CCF-1564000, IIS-1546452 and HDR-1934932.}}
\author[1]{Abolfazl Hashemi}
\author[3]{\\Prateek Jain}
\author[1,2]{Sujay Sanghavi}
\author[1,2]{Inderjit Dhillon}
\author[1]{Ufuk Topcu}
\affil[1]{University of Texas at Austin}
\affil[2]{Amazon Search}
\affil[3]{Google AI}
\maketitle
\begin{abstract}
    Geometric median (\textsc{Gm}) is a classical method in statistics for achieving a robust estimation of the  uncorrupted data; under gross corruption, it achieves the optimal breakdown point of 0.5. However, its computational complexity makes it infeasible for robustifying stochastic gradient descent (SGD) for high-dimensional optimization problems. In this paper, we show that by applying \textsc{Gm} to only a judiciously chosen block of coordinates at a time and using a memory mechanism, one can retain the breakdown point of 0.5 for smooth non-convex problems, with non-asymptotic convergence rates comparable to the SGD with \textsc{Gm}.
\footnote{The implementation of the proposed method is available at 
\href{https://github.com/anishacharya/Optimization-Mavericks}{\textcolor{blue}{Code Link}}}
    \label{section:abstract}
\end{abstract}

\section{Introduction}
\label{section:arx_introduction}
Consider smooth non-convex optimization problems with finite sum structure: 
\begin{equation}
    \label{equation:finite_sum}
    \min_{\mathbf{x}\in \R^d} \left[\bar{f}(\mathbf{x}):=\frac{1}{n}\sum_{i=1}^n f_i(\mathbf{x})\right].
\end{equation}
Mini-batch SGD is the de-facto method for optimizing such functions \cite{robbins1951stochastic, bottou2010large, tsitsiklis1986distributed} which proceeds as follows: at each iteration $t$, it selects a random batch $\gD_t$ of $b$ samples, obtains gradients $\vg_t^{(i)}=\nabla f_i(\vx_t), \;\forall i \in \gD_t$, and updates the parameters using iterations of the form: 
\begin{equation}
\label{equation:sgd}
\vx_{t+1} := \vx_t - \gamma \Tilde{\vg}^{(t)} ,\quad 
\Tilde{\vg}^{(t)} = \frac{1}{|\gD_t|} \sum_{i\in \gD_t} \vg_i^{(t)}.
\end{equation}
In spite of its strong convergence properties in the standard settings \cite{moulines2011non, dekel2012optimal, li2014efficient, goyal2017accurate, keskar2016large, yu2012large}, it is well known that even a small fraction of corrupt samples can lead SGD to an arbitrarily poor solution \cite{bertsimas2011theory, ben2000robust}. 
This has motivated a long line of work to study robust optimization in presence of corruption \cite{blanchard2017machine,alistarh2018byzantine,wu2020federated,xie2019zeno}.
While the problem has been studied under a variety of contamination models, in this paper, we study the robustness properties of the first-order method \eqref{equation:sgd} under the strong and practical \textbf{gross contamination model} (See Definition~\ref{definition:byzantine}) \cite{li2018principled, diakonikolas2019recent, diakonikolas2019sever, diakonikolas2019robust} which also \textit{generalizes the popular Huber's contamination model and the byzantine contamination framework} \cite{huber1992robust, lamport1982byzantine}. \\ \\
In particular, the goal of this work is to design an {\em efficient} first-order optimization method to solve ~\eqref{equation:finite_sum}, which remains {\em robust} even when $0 \leq \psi < 1/2$ fraction of the gradient estimates $\vg_t^{(i)}$ are \textit{arbitrarily corrupted} in each batch $\gD_t$, \textbf{without any prior knowledge about the malicious samples}.
Note that, by letting the corrupt estimates to be \textit{arbitrarily skewed}, this corruption model is able to capture a number of important and practical scenarios including \textbf{corruption in feature} (e.g., existence of outliers) , \textbf{corrupt gradients} (e.g., hardware failure, unreliable communication channels during distributed training) and \textbf{backdoor attacks} \cite{chen2017targeted,liao2018backdoor, gu2019badnets, biggio2012poisoning, mhamdi2018hidden,tran2018spectral} (See Section ~\ref{section:arx_experiment}).
\paragraph{{Robust \textsc{SGD} via \underline{G}eometric \underline{M}edian \underline{D}escent (\textsc{GmD})}.} 
Under the gross corruption model (Definition \ref{definition:byzantine}), the vulnerability of mini-batch SGD can be attributed to the linear gradient aggregation step ~\eqref{equation:sgd} \cite{blanchard2017machine, alistarh2018byzantine, yin2018byzantine, xie2019zeno}. In fact, it can be shown that \textit{no linear gradient aggregation strategy} can tolerate even a \textit{single} grossly corrupted update, i.e., they have a \textit{breakdown point} (Definition \ref{assumption:breakdown_point}) of 0. To see this, consider the single malicious gradient  $\vg_t^{(j)} = -\sum_{i \in {\gD_t \setminus j}}\vg_t^{(i)}$. One can see that this single corrupted gradient results in the average to become $\mathbf{0}$, which in turn means mini-batch SGD gets stuck at the initialization. An approach for robust optimization may be to find an estimate $\Tilde{g}$ such that with high probability $\| \Tilde{g} - \frac{1}{\sG} \sum_{\vg_i \in \sG} \vg_i\|$ is small even in presence of gross-corruption ~\cite{diakonikolas2019recent}.
In this context, geometric median (\textsc{Gm}) (Definition \ref{assumption:geo_med}) is a well studied rotation and translation invariant robust estimator with \textbf{optimal breakdown point} of 1/2 even under gross corruption \cite{lopuhaa1991breakdown, minsker2015geometric, cohen2016geometric}.  
Due to this strong robustness property, SGD with \textsc{Gm}-based gradient aggregation (\textsc{GmD}) has been widely studied in robust optimization literature \cite{alistarh2018byzantine, chen2017distributed, pillutla2019robust, wu2020federated}. Following the notation of ~\eqref{equation:sgd} the update step of \textsc{GmD} can be written as: 
\begin{equation}
\label{equation:gmd}
\vx_{t+1} := \vx_t - \gamma \Tilde{\vg}^{(t)} ,\quad 
\Tilde{\vg}^{(t)} = \textsc{Gm} (\{\vg_i^{(t)}\}) \; \forall i \in [b]
\end{equation}
Despite the strong robustness guarantees of \textsc{Gm}, the computational cost of calculating $\epsilon$ approximate \textsc{Gm} is prohibitive, especially in high dimensional settings. For example, the best known result \cite{cohen2016geometric} uses a subroutine that needs $O(d/\epsilon^2)$ computations to find an $\epsilon$-approximate \textsc{Gm}. 
Despite recent efforts in \cite{vardi2000multivariate,weiszfeld1937point,chandrasekaran1989open,chin2013runtime,cohen2016geometric,pillutla2019robust} to design computationally tractable \textsc{Gm}($\cdot$), given that in practical large-scale optimization settings such as training deep learning models the number of parameters ($d$) is large (e.g.,  $d \approx 60M$ for AlexNet, $d \approx 175B$ for GPT-3).  \textsc{GmD} remains prohibitively expensive \cite{chen2018draco, pillutla2019robust} and with limited applicability. 
\paragraph{Overview of Our Algorithm (\textsc{BgmD}).} 
In this work, we leverage coordinate selection strategies to reduce the cost of \textsc{GmD} and establish \textsc{B}lock co-ordinate \textsc{Gm} \textsc{D}escent (\textsc{BGmD}) (Algorithm ~\ref{algorithm:full}).  \textsc{BGmD} is a robust optimization approach that can significantly reduce the computational overhead of \textsc{Gm} based gradient aggregation resulting in nearly {\em two orders of magnitude} speedup over \textsc{GmD} on standard deep learning training tasks, while maintaining almost the same level of accuracy and optimal breakdown point 1/2 even in presence of gross corruption. 
\\ \\
At a high level, \textsc{BGmD} selects a block of $ 0<k\leq d$ {\em important} coordinates of the gradients; importance of a coordinate is measured according to the largest directional derivative measured by the squared $\ell_2$ norm across all the samples (Algorithm \ref{algorithm:co_sparse_approx}). The remaining $(d-k)$ dimensions are discarded and gradient aggregation happens only along these selected $k$ directions. This Implies the \textsc{Gm} subroutine is performed only over gradient vectors in $\sR^k$ (a significantly lower dimensional subspace). Thus, when $k\ll d$, this approach provides a practical solution to deploy \textsc{Gm}-based aggregation in high dimensional settings\footnote{The notation $k\ll d$ implies that $k$ is at least an order of magnitude smaller than $d$}. 
\\ \\
The motivation behind reducing the \textit{dimension} of the gradient vectors is that in many scenarios most of the information in the gradients is captured by a subset of the coordinates ~\cite{shi2019understanding}. Hence, by the judicious block coordinate selection subroutine outlined in Algorithm \ref{algorithm:co_sparse_approx} one can identify an informative low-dimensional representation of the gradients.
\\ \\
While Algorithm \ref{algorithm:co_sparse_approx} identifies a representative block of the coordinates, \textbf{aggressively reducing the dimension} (i.e., $k\ll d$) might lead to a significant approximation error, which in turn might lead to slower convergence~\cite{nesterov2012efficiency, nutini2015coordinate} (i.e., more iterations), dwarfing the benefit from reduction in per iteration cost. To alleviate this issue, by leveraging the idea of Error Compensation \cite{seide20141, stich2019error, karimireddy2019error} in Algorithm \ref{algorithm:full} we introduce a \textit{memory augmentation} mechanism. Specifically, at each iteration the residual error from dimensionality reduction is computed and accumulated in a memory vector i.e., $\hat{\vm}_t$ and in the subsequent iteration, $\hat{\vm}_t$ is added back to the new gradient estimates. Our ablation studies show that such memory augmentation indeed ensures that the required number of iterations remain relatively small despite choosing $k\ll d$ (See Figure~\ref{figure:choice_k}).
\begin{algorithm}[t]
        \SetAlgoLined
        \textbf{Initialize:}
        estimate: $\vx_0 \in \sR^d$, step-size: $\gamma$, 
        memory: 
        $\hat{\vm}_0 = \textbf{0}$,
        Block Coordinate Selection operator: $\gC_k(\cdot)$, Geometric Median operator: \textsc{Gm}($\cdot$) \\
        \For{epochs \; t = 0, \dots , until convergence}
        {
            Select samples $\gD_t=\{i_1, \dots, i_b\}$ \\
            Obtain : $\vg_t^{(i)} := \nabla f_i(\vx_t) , \; \forall i \in \gD_t$\\
             Let $\mG_t \in \sR^{b \times d}$ s.t. each row $\mG_t[i, :] = \vg_t^{(i)}$\\
             $\mG_t[i, :] = \gamma \mG_t[i, :] + \hat{\vm}_t\; \forall i \in [b]$ \quad (Memory Augmentation)\\
             $\boldsymbol{\Delta}_t := \mathcal{C}_k (\mG_t)$ \quad (Select top $k$ columns via Algo \ref{algorithm:co_sparse_approx})\\
             $\mM_{t} = \mG_t - \boldsymbol{\Delta}_t $ \quad (Compute Residuals)\\
             $\hat{\vm}_{t+1} = \frac{1}{b}\sum_{0\leq i \leq b}\mM_t[i, :]$
             \quad (Update memory)\\
             $\Tilde{\vg}_t := \textsc{Gm}(\boldsymbol{\Delta}_t)$ \quad(Robust Aggregation in $\sR^k$)\\
             $\vx_{t+1}:= \vx_t - \Tilde{\vg}_t$ \quad (Global model update)\\
        }
        \caption{Block \textsc{Gm} Descent (\textsc{BGmD})}
        \label{algorithm:full}
\end{algorithm}
\begin{algorithm}[t]

    \SetAlgoLined
    \textbf{Input:} \; $\mG_t \in \sR^{n \times d},\; k $ \\
    \For{coordinates \; j = 0, \dots , d-1}
    {
        $s_j \gets \|\mG_t[:, j]\|^2$ \; (norm along each dimension)
    }
    Sample set $\sI_k$ of $k$ dimensions with probabilities proportional to $s_j$\\
    $\gC_k(\mG)_{i, j\in \sI_k} =  \mG_{i, j} , \; 
        \gC_k(\mG)_{i, j\notin \sI_k} = 0 $ \\
    \textbf{Return:} $\gC_k(\mG)$
    \caption{Block Coordinate Selection Strategy}
    \label{algorithm:co_sparse_approx}
\end{algorithm}

\paragraph{Contributions.} 
The main contributions of this work are as follows: 
\begin{itemize}\setlength \itemsep{0.05em}
\item We propose \textsc{BGmD} (Algorithm \ref{algorithm:full}), a method for robust optimization in high dimensions.  
\textsc{BGmD} is significantly more efficient than the standard \textsc{Gm}-SGD method but is still able to maintain the optimal breakdown point 1/2 -- first such efficient method with provable convergence guarantees. 

\item We provide strong guarantees on the rate of convergence of \textsc{BGmD} in standard non-convex scenarios including smooth non-convex functions, and non-convex functions satisfying the Polyak-\L{}ojasiewicz Condition. These rates are comparable to those established for \textsc{Gm}-SGD under more restricting conditions such as strong convexity \cite{ chen2017distributed,pillutla2019robust,wu2020federated}.
\item Through extensive 
experiments under several common corruption settings ; we demonstrate that \textsc{BGmD} can be up to $3\times$ more efficient to train than \textsc{Gm}-SGD on  Fashion MNIST and CIFAR-10 benchmarks while still ensuring  similar test accuracy and maintaining same level of robustness.  

\end{itemize}


\begin{table}[t]
\centering
\caption{Comparison of time complexity and robustness properties of different robust optimization methods. The bold quantities show a method achieves the theoretical limits. $*$ \textsc{CmD} throughout the paper will refer to Co-ordinate wise median descent i.e. simply replacing the aggregation step of \textsc{SGD} by \textsc{Cm}.}
\vspace{1mm}
\begin{tabular}{@{}ccc@{}}
\toprule
Algorithm & Iteration Complexity & Breakdown Point \\ \midrule
SGD 
&\multicolumn{1}{c}{\pmb{$\gO(bd)$}} 
&\multicolumn{1}{c}{0}\\
CMD$^*$ ~\cite{yang2019byzantine, yin2018byzantine}
&\multicolumn{1}{c}{\pmb{$\gO(bd)$}}          
&\multicolumn{1}{c}{1/2 - $\Omega(\sqrt{d/b})$}  \\
GMD ~\cite{cohen2016geometric, wu2020federated}
&\multicolumn{1}{c}{$\gO(d \epsilon^{-2} + bd)$}
& \multicolumn{1}{c}{\textbf{1/2}} \\
\cite{data2020byzantine}
&\multicolumn{1}{c}{$\gO(d b^2 \min(d , b))$}
& \multicolumn{1}{c}{1/4}  \\ 
\textsc{BGmD} (This Work) 
&\multicolumn{1}{c}{$\gO(k \epsilon^{-2} + bd)$}
& \multicolumn{1}{c}{\textbf{1/2}} \\
\bottomrule
\end{tabular}
\label{tab:compare}
\end{table}

\section{Related Work}
\label{section:arx_related_work}
\paragraph{\textbf{Computationally Tractable Robust SGD.}}
Robust optimization in the presence of gross corruption has received renewed impetus in the machine learning community, following practical considerations such as preserving the privacy of the user data and coping with the existence of adversarial disturbances \cite{blanchard2017machine, chen2017distributed, mcmahan2017communication}. There are two main research directions in this area.
\\
The first direction aims at designing robustness criteria to identify and subsequently filter out corrupt samples before employing the linear gradient aggregation technique in \eqref{equation:sgd}. For example, \cite{ghosh2019communication, gupta2020byzantine} remove the samples with gradient norms exceeding a predetermined threshold, \cite{yin2018byzantine} remove a fraction of samples from both tails of the gradient norm distribution, \cite{chen2018draco, yang2019bridge} use redundancy \cite{von1956probabilistic} and majority vote operations, \cite{diakonikolas2019sever} rely on spectral filtering, \cite{steinhardt2017resilience, blanchard2017machine, data2020byzantine, bulusu2020distributed} use $(\epsilon, \sigma)$-resilience based iterative filtering approach, while \cite{xie2019zeno} instead add a resilience-based penalty to the optimization task to implicitly remove the corrupt samples.  
\\
Our approach falls under the second research direction, where the aim is to replace mini-batch averaging with a \textbf{robust gradient aggregation operator}. In addition to \textsc{Gm} operator  \cite{feng2014distributed, alistarh2018byzantine, chen2017distributed} which was discussed earlier, other notable examples of robust aggregation techniques include krum \cite{blanchard2017machine}, coordinate wise median \cite{yin2018byzantine}, and trimmed mean \cite{yin2018byzantine}. Despite these alternatives, \textsc{Gm} is the most resilient method in practice achieving the optimal breakdown point of 1/2. We compare a number of robust optimization methods in this direction with \textsc{BGmD} in terms of computational complexity and breakdown in Table \ref{tab:compare}. 
\paragraph{\textbf{Connection to Coordinate Descent}}
Coordinate Descent (CD) refers to a class of methods wherein at each iteration, a block of coordinates is chosen and subsequently updated using a descent step. While this strategy has a long history~\cite{sardy2000block, fu1998penalized, bertsekas1997nonlinear, joachims1998making, optimization1998fast}, it has received renewed interest in the context of modern large scale machine learning with very high dimensional parameter space \cite{nesterov2012efficiency, richtarik2014iteration, tseng2009block, tseng2009coordinate, stich2017approximate, beck2013convergence}. There has been significant research efforts focusing on efficient coordinate selection strategy. For instance, \cite{nesterov2012efficiency, needell2014paved, shalev2013accelerated, richtarik2014iteration, allen2016even} choose the coordinates randomly while \cite{saha2013nonasymptotic, gurbuzbalaban2017cyclic} do so cyclically in a fixed order, referred as the Gauss Seidel rule, and \cite{nutini2015coordinate, nutini2017let, hsieh2011fast, dhillon2011nearest, you2016asynchronous, karimireddy2019efficient} propose choosing the coordinates greedily  according to norm-based selection criteria, a strategy known as the Gauss Southwell rule.

\begin{remark}
Note that, our approach (Algorithm~\ref{algorithm:co_sparse_approx}) is closely related to the Greedy Gauss Southwell coordinate selection approach. In fact, it is immediate that for batch size $b=1$ Gauss Southwell Co-ordinate Descent becomes a special case of \textsc{BGmD}. 
\end{remark}

\paragraph{\textbf{Connection to Error Feedback}}
Compensating for the loss incurred due to approximation through a memory mechanism is a common concept in the feedback control and signal processing literature (See \cite{doyle2013feedback} and references therein). \cite{seide20141, strom2015scalable} adapt this to gradient compression (1Bit-SGD) to reduce the number of communicated bits in distributed optimization. Recently, \cite{stich2018sparsified, stich2019error, karimireddy2019error} have analyzed this error feedback framework for a number of gradient compressors in the context of communication-constrained distributed training. 
\begin{remark}
Note that our memory mechanism is inspired by this error feedback mechanism. In fact, all the works on error feedback to compensate for gradient compression\textup{~\cite{stich2018sparsified, stich2019error, karimireddy2019error}} are special case of our proposed memory mechanism when batch size $b=1$ i.e. $\mM_t = \hat{\vm}_t$. 
\end{remark}

\section{Preliminaries}
\label{section:arx_prelims}
We first briefly recall some related concepts  and  state our main assumptions. Throughout, $\|.\|$ denotes the $\ell_2$ norm unless otherwise specified. 
\begin{definition}[{\textbf{Gross Corruption Model}}]
\label{definition:byzantine}
Given $0 \leq \psi < \frac{1}{2}$ and a distribution family $\gD$ on $\sR^d$ the adversary operates as follows: $n$ samples are drawn from $D \in \gD$. The adversary is allowed to \textbf{inspect} all the samples and replace up to $\psi n$ samples with arbitrary points. 

\textup{Intuitively, this implies that $(1 - \psi)$ fraction of the training samples are generated from the true distribution (\textit{inliers}) and rest are allowed to be \textbf{arbitrarily corrupted} (\textit{outliers}) i.e. $\alpha:=|\sB|/ |\sG|= \frac{\psi}{1 - \psi} < 1$, where $\sB$ and $\sG$ are the sets of corrupt and good samples. In the rest of the paper, we will refer to a set of samples generated through this process as \textbf{$\alpha$-corrupted}.}
\end{definition}
\begin{definition}[{\textbf{Breakdown Point}}]
\label{assumption:breakdown_point}
Finite-sample breakdown point \textup{\cite{donoho1983notion}} is a way to measure the resilience of an estimator. It is defined as the smallest fraction of contamination that must be introduced to cause an estimator to break i.e. produce arbitrarily wrong estimates.

\textup{In the context Definition \ref{definition:byzantine} we can say an estimator has the optimal breakdown point 1/2 if it is robust in presence of $\alpha$-corruption $\forall \; \psi < 1/2$ or alternatively $\forall \; \alpha < 1$}.
\end{definition}

\begin{definition}[{\textbf{Geometric Median}}]
\label{assumption:geo_med}
Given a finite collection of observations $\vx_1, \vx_2, \dots \vx_n$ defined over 
a separable Hilbert space $\mathbb{X}$ with norm $\|\cdot\|$ the geometric median or the Fermet-Weber point \textup{\cite{haldane1948note,weber1929alfred,minsker2015geometric,kemperman1987median}} is defined as: \vspace*{-5pt}
\begin{equation}
\label{equation:gm}
    \vx_* = \textsc{Gm}(\{\vx_i\}) =\argmin_{\vy\in \mathbb{X}} \bigg[g(\vx):=\sum_{i=1}^{n}\|\vy - \vx_i\|\bigg]
\end{equation}
We call a point $\vx \in \R^d$ an $\epsilon$-accurate geometric
median if $g(\vx) \leq (1+\epsilon)g(\vx_*)$.
\end{definition}
\begin{assumption}[{\textbf{Stochastic Oracle}}]
\label{assumption:bounded_g}
Each non-corrupt sample $i\in \sG$ is endowed with an unbiased stochastic first-order oracle with bounded variance. That is,  
\begin{equation}
\label{equation:unbiased_g}
    \E_{z\sim\gD_i}[\vg_i(\vx,z)] = \nabla f_i(\vx),
    \qquad\E_{z\sim\gD_i}\|\nabla F_i(\vx,z)\|^2\leq \sigma^2
\end{equation}
\end{assumption}
\begin{assumption}[
\hspace{-0.1 cm}{\textbf{Smoothness}}]
\label{assumption:smooth}
Each non-corrupt function $f_i$ is $L$-smooth, $\forall i \in \sG$,  
\begin{equation}
f_i(\vx)\leq f_i(\vy)+\langle\vx-\vy, \nabla f_i(\vy)\rangle+\frac{L}{2}\|\vx-\vy\|^2.  \quad \forall \vx,\vy \in \sR^d.
\end{equation}
\end{assumption}
\begin{assumption}[
\hspace{-0.1 cm}{\textbf{Polyak-\L{}ojasiewicz Condition}}]
\label{assumption:plc}
The average of non-corrupt functions $f:=\frac{1}{\sG}\sum_{i\in \sG} f_i(\mathbf{x})$ 
satisfies the Polyak-\L{}ojasiewicz condition (PLC) with parameter $\mu$, i.e. 
\begin{equation}
\begin{aligned}
&\|\nabla f(\vx)\|^2 \geq 2\mu (f(\vx)-f(\vx^\ast)), \; \mu>0 \quad \textup{where} 
&\vx^\ast = \argmin_\vx f(\vx),\quad \forall \vx \in \sR^d.
\end{aligned}
\end{equation}
We further assume that the solution set $\gX^\ast \in \sR^d$ is non-empty and convex. Also note that $\mu < L$.
\end{assumption}
Simply put, the Polyak-\L{}ojasiewicz condition implies that when multiple global optima exist, each stationary point of the objective function is a global optimum \cite{polyak1963gradient, karimi2016linear}. This setting enables the study of modern large-scale ML tasks that are generally non-convex. Also note that $\mu$-strongly convex functions satisfy PLC with parameter $\mu$ implying \textit{PLC is a weaker assumption than strong convexity}. 

\section{Block Coordinate Geometric Median Descent (\textsc{BGmD})}
\label{section:arx_algorithm}
As discussed earlier, \textsc{BGmD} (Algorithm \ref{algorithm:full}) involves two key steps: (i) Selecting a block of coordinates  and  run computationally expensive \textsc{Gm} aggregation over a low dimensional subspace and (ii) Compensating for the residual error due to block coordinate selection. In the rest of this section we will discuss these two ideas in more detail.
\paragraph{\textbf{Block Selection Strategy.}}
\label{section:algo_selection}
The key intuition why we might be able to select a small number of $k$ coordinates for robust mean estimation is that in practical over-parameterized models, the main mass of the gradient is concentrated in a few coordinates  \cite{chaudhari2019entropy}. So, what would be the {\em best strategy to select the most informative block of coordinates?} Ideally, one would like to select the best $k$ dimensions that would result in the largest decrease in training loss. However, this task is NP-hard in general \cite{das2011submodular,nemhauser1981maximizing, charikar2000combinatorial}.  Instead, we adopt a simple and fast block coordinate selection rule: 
Consider $\mG \in \sR^{b \times d}$ where each row corresponds to the transpose of the stochastic gradient estimate, i.e. $\mG_{i,:} = \vg_i^T \in \sR^{1 \times d},\quad \forall i \in [b]$ where $b$ is number of samples (or batches/clients in case of distributed setting) participating in that iteration \eqref{equation:sgd}. Then, selecting $k$ dimensions is equivalent to selecting $k$ columns of $\mG$, which we select according to the norm of the columns. That is,  we assign a score to each dimension proportional to the $\ell_2$ norm (total mass along that coordinate) i.e. $s_j = \|\mG_{:, j}\|^2,$ for all $ j\in [d]$. We then sample only $k$ coordinates with with probabilities proportional to $s_j$ 
and discard the rest to find a set $\omega_k$  of size $k$ (see Algorithm~\ref{algorithm:co_sparse_approx}). 
We show below that this method produces a contraction approximation to $\mG$. 
\begin{lemma}
\label{lemma:norm_sampling}
Algorithm \ref{algorithm:co_sparse_approx} yields a contraction approximation, i.e., $\E_{\gC_k} \left[\|\gC_k(\mG)-\mG\|_F^2 |\vx\right] \leq (1-\xi)\|\mG\|_F^2 , \; \frac{k}{d} \leq \xi \leq 1,$ 
where $\gC_k(\mG)_{i, j\in \omega_k} = \mG_{i, j},$ and $\gC_k(\mG)_{i, j\notin \omega_k} =0$.
\end{lemma}
\textit{It is also worth noting that without additional distributional assumption on $\mG$ the lower bound on $\xi$
cannot be improved. \footnote{To see this, consider the case where each $\vg_t^{(i)}$ is uniformly distributed along each coordinates. Then, the algorithm would satisfy Lemma \ref{lemma:norm_sampling} with $\xi = \frac{k}{d}$. In this scenario, the achievable bound is identical to the bound achieved via choosing the $k$ dimensions uniformly at random.}
However, in practice the gradient vectors are extremely unlikely to be uniform \textup{\cite{alistarh2018convergence}} and thus active norm sampling is expected to to satisfy Lemma \ref{lemma:norm_sampling}} with $\xi \approx 1$.

\paragraph{\textbf{The Memory Mechanism.}}
While descending along only a small subset of $k$ coordinates at each iteration  significantly improves the per iteration computational cost, a smaller value of $k$ would also imply larger gradient information loss i.e., 
a smaller $\xi$ (Lemma \ref{lemma:norm_sampling}). Intuitively, a restriction to a $k$-dimensional subspace results in a $\frac{d}{k}$ factor increase in the gradient variance \cite{stich2018sparsified}.
To mitigate this, we adopt a memory mechanism 
\cite{seide20141, strom2015scalable, stich2018sparsified, karimireddy2019error, stich2019error,doyle2013feedback} 
Our approach is as follows: \quad Throughout the training process, we keep track of the residual errors in $\mG_t-\gC_k(\mG_t)$ via $\hat{\vm}_t \in \sR^d$ that we call memory. At each iteration $t$, it simply accumulates the residual error incurred due to ignoring $d-k$ dimensions, averaged over all the samples participating in that round. In the next iteration, \(\hat{\vm}_t\) is added back to all the the new gradient estimates as feedback. Following our Jacobian notation, the memory update is given as: 
\begin{equation}
    \begin{aligned}
    & \textup{\textbf{Memory Augmentation}} \quad \mG_t[i, :] = \gamma \mG_t[i, :] + \hat{\vm}_t\; \forall i \in [b]\\
    & \textup{\textbf{Memory Update}} \quad \mM_t = \mG_t - \gC_k (\mG_t) \;; \hat{\vm}_{t+1} = \frac{1}{b}\sum_{0\leq i \leq b}\mM_t[i, :]
    \end{aligned}
\end{equation}
Intuitively, as the residual at each iteration is not discarded but  rather kept in memory and added back in a future iteration, this ensures similar convergence rates as training in $\sR^d$ (see Theorem \ref{theorem:nonconvex-simple}). 

\section{Computational Complexity Analysis}
\label{section:arx_comp_complexity}
To theoretically understand the overall computational benefit of our proposed scheme \textsc{BGmD} over \textsc{Gm-SGD}, we first analyze the per epoch time complexity of both the algorithms. 
Combining this with the convergence rates derived in Theorems \ref{theorem:nonconvex-simple} and \ref{theorem:plc-simple} provides a clear picture of the overall computational efficiency of \textsc{BGmD}.  

Consider  problem \eqref{equation:finite_sum} with parameters $\vx \in \sR^d$  and SGD style iterations of the form \eqref{equation:sgd} and batch size $|\gD_t| = b$. Note that the difference between the iterations of \textsc{SGD}, \textsc{Gm-SGD}, and \textsc{BGmD} is primarily in the aggregation step, i.e., how they aggregate the updates communicated by samples (or batches) participating in training during that iteration. 

At iteration $t$, let $T_a^{(t)}$ denote the time to aggregate the gradients. Also let, $T_b^{(t)}$ denote the time taken to compute the batch gradients (i.e., time to perform back propagation). Thus, the overall complexity of one training iteration is roughly $\gO(T_a^{(t)} + T_b^{(t)})$. Now, note that $T_b^{(t)}$ is approximately the same for all the algorithms mentioned above and for methods like \textsc{Gm-SGD} $T_b^{(t)}\ll T_a^{(t)}$. So we focus on $T_a^{(t)}$ to study the relative computational cost of \textsc{SGD}, \textsc{Gm-SGD}, and \textsc{BGmD}. To compute $\Tilde{\vg}^{(t)}$ vanilla \textsc{SGD} needs to compute average of $b$ gradients in $\sR^d$ implying $\gO(bd)$ cost. On the other had, \textsc{Gm-SGD} and its variants \cite{alistarh2018byzantine, chen2017distributed, byrd2012sample} require computing $\epsilon$-approximate \textsc{Gm} of $b$ points in $\sR^d$ incurring per iteration cost of at least $\gO(\frac{d}{\epsilon^2})$ \cite{cohen2016geometric, pillutla2019robust,alistarh2018byzantine, chen2017distributed}. In contrast, $T_a^{(t)}$ for \textsc{BGmD} is $\gO(\frac{k}{\epsilon^2}+bd)$, where the first term in computational complexity is due to computation of  \textsc{Gm} of $\sR^k$-dimensional points. The second term is due to the coordinate sampling procedure. 

That is, by  choosing a sufficiently small block of coordinates i.e. $k \ll d$, \textsc{BGmD} can result in significant savings per iteration. Based on this observation, one can derive the following Lemma.  



\begin{lemma}
\label{proposition:compute}\textup{\textbf{(Choice of k).}}
Let $k \leq \gO(\frac{1}{F} - b\epsilon^2)\cdot d$. Then, given an $\epsilon$- approximate \textsc{Gm} oracle,  Algorithm \ref{algorithm:full} achieves a factor $F$ speedup over \textsc{Gm-SGD} for aggregating $b$ samples. 
\end{lemma}
In most practical settings, the second term of $b\epsilon^2$ in Lemma~\ref{proposition:compute} is often negligible compared to the first term which implies significantly {\em smaller per-iteration} complexity for \textsc{BGmD} compared to that of \textsc{Gm-SGD}. Interestingly, in Theorem~\ref{theorem:nonconvex-simple}  we show that the rate of convergence and break-down point for both the methods is similar, which implies that overall \textsc{BGmD} is much more efficient in optimizing \eqref{equation:finite_sum} than \textsc{Gm-SGD}.


\section{Convergence Guarantees of \textsc{BGmD}}
\label{section:arx_analysis}
We now analyze the convergence properties of \textsc{BGmD} as described in Algorithm \ref{algorithm:full}. We state the results in Theorem \ref{theorem:nonconvex-simple} and Theorem \ref{theorem:plc-simple} for general non-convex functions and functions satisfying PLC, respectively. 

\begin{theorem}
[\textbf{Smooth Non-convex}]
\label{theorem:nonconvex-simple}
Consider the general case where the functions $f_i$ correspond to non-corrupt samples $i \in \sG$ are \textbf{non-convex} and \textbf{smooth} (Assumption \ref{assumption:smooth}). Define, $R_0 := f(\vx_0)-f(\vx^\ast)$ where $\vx^\ast$ is the true optima and $\vx_0$ is the initial parameters. Run Algorithm \ref{algorithm:full} with compression factor $ \frac{k}{d} \leq \xi \leq 1$ (Lemma \ref{lemma:norm_sampling}), learning rate $\gamma =  1/2L$ and $\epsilon-$approximate \textup{\textsc{Gm}($\cdot$)} oracle in presence of $\alpha-$corruption (Definition \ref{definition:byzantine}) for $T$ iterations. Sample an iteration $\tau$ from $1,\dots, T$ uniformly at random. Then, it holds that:
\begin{equation}\nonumber
\label{eq:thm:nonconvex}
\begin{aligned}
\E\|\nabla f(\vx_{\tau})\|^2 = \gO\Bigg( \frac{L R_0}{T}+
\frac{\sigma^2\xi^{-2}}{(1-\alpha)^2}+\frac{L^2\epsilon^2}{|\sG|^2(1-\alpha)^2}
\Bigg).
    \end{aligned}
\end{equation}
\end{theorem}
\begin{theorem}
[\textbf{Non-convex under PLC}]\label{theorem:plc-simple}
State the notation of Theorem \ref{theorem:nonconvex-simple}. Assume
$f=\frac{1}{\sG}\sum_{i\in \sG} f_i(\mathbf{x})$ satisfies the \textbf{Polyak-\L{}ojasiewicz Condition} (Assumption \ref{assumption:plc}) with parameter $\mu$. 
After $T$ iterations Algorithm \ref{algorithm:full} with compression factor $ \frac{k}{d} \leq \xi \leq 1$ (Lemma \ref{lemma:norm_sampling}), learning rate $\gamma =  1/4L$ and $\epsilon-$approximate \textup{\textsc{Gm}($\cdot$)} oracle in presence of $\alpha-$corruption (Definition \ref{definition:byzantine}) satisfies:
\begin{equation}
\begin{aligned}
\E\|\hat{\vx}_T-\vx^\ast\|^2 &= \gO\Bigg( \frac{LR_0}{\mu^2}\left[1-\frac{\mu}{8L}\right]^T+\frac{\sigma^2\xi^{-2}}{\mu^2(1-\alpha)^2} + \frac{L^2\epsilon^2}{\mu^2 |\sG|^2(1-\alpha)^2}\Bigg),
 \end{aligned}
 \end{equation}
 for a global optimal solution  $\vx^\ast \in \gX^\ast$. Here,
 $\hat{\vx}_T := \frac{1}{W_T}\sum_{t=0}^{T-1}w_t\vx_t$ with weights $w_t := (1-\frac{\mu}{8L})^{-(t+1)}$, $W_T := \sum_{t=0}^{T-1}w_t$.
\end{theorem}
Theorem \ref{theorem:nonconvex-simple} and Theorem \ref{theorem:plc-simple} state that Algorithm \ref{algorithm:full} with a constant stepsize  convergences to a neighborhood of a first order stationary point. The radius of this neighborhood depends on two terms. The first term depends on the variance of the stochastic gradients as well as the effectiveness of the coordinate selection strategy through $\xi$. The second term however depends on how accurate the \textsc{Gm}  computation is performed in each iteration. As we demonstrate in our experiments typically the proposed coordinate selection strategy entails a $\xi \approx 1$. Therefore, we expect the radius of error to be negligible. Furthermore, both terms in the radius depend on the fraction of corrupted samples and as long as less than 50\% of the samples are corrupted, \textsc{BGmD} attains convergence to a neighborhood of a stationary  point. We also note that this rate matches the rate of \textsc{Gm-SGD} when the data is not i.i.d. (see e.g. \cite{chen2017distributed,alistarh2018byzantine,wu2020federated,data2020byzantine} and the references therein) while Algorithm \ref{algorithm:full} significantly reduces the cost of \textsc{Gm}-based aggregation. We further note that the convergence properties established in Theorem \ref{theorem:plc-simple} are analogous to those of \textsc{Gm-SGD}. However, compared to the existing analyses that require strong convexity(\cite{alistarh2018byzantine, wu2020federated, data2020byzantine}), Theorem \ref{theorem:plc-simple} only assumes PLC which as we discussed is a much milder condition. Furthermore, in absence of corruption (i.e., $\alpha = 0$), if the  data is i.i.d., our theoretical analysis reveals that by setting $\gamma = \mathcal{O}(1/\sqrt{T})$ Algorithm \ref{algorithm:full} convergences at the rate of $\gO(1/\sqrt{T})$ to the statistical accuracy.\footnote{setting $\gamma = \mathcal{O}(1/T)$, \textsc{BGmD} convergences at the rate of $\gO(\frac{\log T}{T})$ under PLC and no corruption.} This last result can be established by using the concentration of the median-of-the-means estimator \cite{chen2017distributed}.

\subsection{Proof Outline}
\label{section:arx_proof_outline}
Following \cite{stich2018local,karimireddy2019error,basu2019qsparse}, we start by defining a sequences of averaged quantities. Divergent from these works however, due to the adversarial setting being considered, we define these quantities over only the uncorrupted samples:
\begin{equation}
\label{equation:avgs}
    \begin{aligned}
        \vg_t &= \frac{1}{|\sG|} \sum_{i\in\sG} \vg_t^i,\quad  
        \bar{\vg}_t = \E_t[\vg_t] =\frac{1}{|\sG|} \sum_{i\in\sG} \nabla f_i(\vx_t),\\
        \vm_t &= \frac{1}{|\sG|} \sum_{i\in\sG} \mM_t[i, :],\quad
       \Delta_t = \frac{1}{|\sG|} \sum_{i\in\sG} \Delta_t^i,\quad
        \vp_t = \frac{1}{|\sG|} \sum_{i\in\sG} \vp_t^i =  \gamma_t\vg_t+\vm_t.
    \end{aligned}
\end{equation}
Notice that the \textsc{BGmD} cannot compute the above average sequences over the uncorrupted samples, but it aims to approximate the aggregated stochastic gradients of the reliable clients, i.e. $\vg_t$, via the \textsc{Gm}($\cdot$) oracle. Noting the definition of $\Delta_t$ in \eqref{equation:avgs}, the update can be thought of as being a perturbed sequence with the perturbation quantity $\vz_t :=  \Tilde{\vg}_t -\textsc{Gm}(\{\Delta_t^i\}_{i\in \gG}).$
Next, using the closeness of the \textsc{Gm}($\cdot$) oracle to the true average we will establish the following lemmas to bound the perturbation. 
\begin{lemma}[\textbf{Bounding the Memory}]
\label{lemma:bound-e}
Consider the setting of Algorithm \ref{algorithm:full} in iteration $t$ with compression factor $\xi$ (Lemma \ref{lemma:norm_sampling}), learning rate $\gamma$ and in presence of $\alpha-$corruption (Definition \ref{definition:byzantine}). If further $f_i$ have bounded variance $\sigma^2$ (Assumption \ref{assumption:bounded_g}) we have
\begin{equation}
\label{equation:bound-e}
    \E\|\hat{\vm}_t\|^2\leq 4(1-\xi^2)\gamma^2\sigma^2\xi^{-2}, \quad \forall i \in [n].
\end{equation}
\end{lemma}
\begin{lemma}[\bf Bounding the Perturbation]
\label{lemma:bound-z-simplified}
Consider the setting of Algorithm \ref{algorithm:full} in iteration $t$ with compression factor $\xi$ (Lemma \ref{lemma:norm_sampling}), learning rate $\gamma$ and $\epsilon-$approximate \textup{\textsc{Gm}($\cdot$)} oracle in presence of $\alpha-$corruption (Definition \ref{definition:byzantine}). Under the assumption that function $f_i$ are smooth (Assumption \ref{assumption:smooth}),
\begin{equation}\label{eq:bound-z}
    \E\|\vz_t\|^2 \leq \frac{96\gamma^2\sigma^2}{(1-\alpha)^2}\left[1+\frac{4(1-\xi^2)}{\xi^2}\right]+\frac{2\epsilon^2}{|\sG|^2(1-\alpha)^2}.
\end{equation}
\end{lemma}
With the bound in Lemma \ref{lemma:bound-z-simplified}, we define the following perturbed virtual sequences for $i\in \sG$ 
\begin{equation}
        \tilde{\vx}_{t+1}^i = \tilde{\vx}_{t}- \gamma \vg_t^i -\vz_t, \quad \tilde{\vx}_{0}^i = \vx_0,\quad \tilde{\vx}_{t+1}= \frac{1}{|\sG|}\sum_{i\in \sG}\tilde{\vx}_{t+1}^i = \tilde{\vx}_{t}- \gamma \vg_t -\vz_t
\end{equation}
Again, $\tilde{\vx}_{t}$ can be though of as a perturbed version of the SGD iterates over only the good samples $i\in \sG$.
Notice that \textsc{BGmD} does not compute the virtual sequence and this sequence is defined merely for the theoretical analysis. Therefore, it is essential to establish its relation to ${\vx}_{t}$, i.e. the iterates of \textsc{BGmD}. We do so in Lemma \ref{lemma:bound-virtual-simplified}. 
\begin{lemma}[\bf Memory as a Delayed Sequence]
\label{lemma:bound-virtual-simplified}
Consider the setting of Algorithm \ref{algorithm:full} in iteration $t$. It holds that $\vx_t - \tilde{\vx}_t = \vm_t.$
\end{lemma}
The main challenge in showing this result is the presence of perturbations $\vz_t$ in the resilient aggregation that we adopt in Algorithm \ref{algorithm:full}.
Upon establishing this lemma, using smoothness we establish a bound on suboptimality of the  model learned at each iteration of \textsc{BGmD} as a function of the perturbed virtual sequence $\tilde{\vx}_t$.
\begin{lemma}
[\bf Recursive Bounding of the Suboptimality]
\label{lemma:recursive_bound}
For any $0<\rho<0.5$ it holds that
\begin{equation}
\begin{aligned}
\E_t[f(\tilde{\vx}_{t+1})] &\leq f(\tilde{\vx}_{t})-\left(\frac{1}{2}-\rho\right)\frac{\gamma}{2} \|  \nabla f(\vx_t)\|^2+\frac{3\gamma L^2}{2}\|\tilde{\vx}_{t}-\vx_t\|^2\\
&\qquad\qquad+L\gamma^2\E_t\|\vg_t\|^2+ \left(L+\frac{1}{2\rho\gamma}+\frac{1}{2\gamma}\right)\E_t\|\vz_t\|^2.
    \end{aligned}
\end{equation}
\end{lemma}
The proofs are furnished by noting that the amount of perturbation can be bounded by using smoothness and the PLC assumptions.


\section{Empirical Evidence}
\label{section:arx_experiment}
In this section, we describe our experimental setup and present our empirical findings and establish strong insights about the performance of \textsc{BgmD}. Table~\ref{table:fmnist},~\ref{table:mnist} and~\ref{table:cifar} provide a summary of the results.

\subsection{Experimental Setup}
We use the following two important optimization setups for our experiments: 

\paragraph{\textbf{Homogeneous Samples}} We trained a moderately large LeNet \cite{lecun1998gradient} style CNN with 1.16M parameters on the challenging Fashion-MNIST \cite{xiao2017fashion} dataset in the \textit{vanilla mini batch} setting with 32 parallel mini-batches each with batch size 64. The training data was \textbf{i.i.d} among all the batches at each epoch. We use initial learning rate of 0.01 which was decayed by $1\%$ at each epoch. Each experiment under this setting was run for \textbf{50 epochs} (full passes over training data). 
We also train the same model on simple MNIST~\cite{lecun1998gradient} dataset using mini-batch SGD with 10 parallel batches each of size 64. We use a learning rate 0.001 updated using cosine annealing scheduler along with a momentum 0.9. 

\paragraph{\textbf{Heterogeneous Samples}} Our theoretical results (Theorem ~\ref{theorem:nonconvex-simple},~\ref{theorem:plc-simple}) are established without any assumption on how the data is distributed across batches. We verify this by training an 18-layer wide ResNet \cite{he2016deep} with 11.2M parameters on CIFAR 10 \cite{krizhevsky2012imagenet} in a \textbf{federated learning} setting \cite{mcmahan2017communication}. In this setting, at each iteration 10 parallel clients compute gradients over batch sizes of 128 and communicate with the central server. The training data was distributed among the clients in a \textbf{non i.i.d} manner only once at the beginning of training. Each experiment was run for \textbf{200 epochs} with an initial learning rate 0.1 , warm restarted via cosine annealing. \\
To \textbf{Simulate heterogeneous (non i.i.d) data} distribution among clients we use the following scheme:  The entire training data was first sorted based on labels and then contiguously divided into 80 equal data-shards $\gD_i$ i.e. for Fashion-MNIST each data shard was assigned 750 samples. Each client was then randomly assigned 8 shards of data at the beginning of training implying with high probability no client had access to all classes of data ensuring heterogeneity~\cite{li2018federated, li2019convergence, das2020faster, karimireddy2020scaffold}.
\\ \\
Common to all the experiments, we used the categorical cross entropy loss plus an $\ell_2$ regularizer with weight decay value of $1\mathrm{e}{-4}$. All the experiments are performed on a single 12 \textsc{GB Titan Xp} GPU. For reproducibility, all the experiments are run with \textbf{deterministic \textsc{CuDNN} back-end}. Further, each experiment is repeated 5 times with different random seeds and confidence interval is noted in the Tables~\ref{table:fmnist}~\ref{table:cifar} and ~\ref{table:mnist}.  

\begin{table}[htb!]
\centering
\caption{\textbf{Homogeneous Training:} 1.12M parameter CNN trained on Fashion MNIST in regular i.i.d. setting. For all corruption types, test accuracy of \textsc{BGmD} is similar to that of \textsc{GmD} and surprisingly, in some cases even higher. As expected, \textsc{SGD} often diverges under corruption (empty entries in the table). \textsc{CmD} performs sub-optimally as corruption is increased.}
\begin{tabular}{llllll}
\toprule
& Corruption (\%) 
& \multicolumn{1}{c}{SGD} 
& \multicolumn{1}{c}{\textsc{CmD}} 
& \multicolumn{1}{c}{\textsc{BGmD}} 
& \multicolumn{1}{c}{\textsc{GmD}} \\ \midrule
Clean             
& \multicolumn{1}{c}{-}   
& \multicolumn{1}{c}{$\mathbf{89.39}\textcolor{gray}{\pm 0.28}$}    
& \multicolumn{1}{c}{$83.82\textcolor{gray}{\pm 0.26}$}    
& \multicolumn{1}{c}{$89.25\textcolor{gray}{\pm 0.19}$}     
& \multicolumn{1}{c}{$88.98\textcolor{gray}{\pm 0.3}$}    \\
\multicolumn{6}{c}{}\\
\multicolumn{6}{c}{Gradient Attack}\\
\multicolumn{6}{c}{}\\
\multirow{2}{*}{Bit Flip} 
& \multicolumn{1}{c}{20}
& \multicolumn{1}{c}{-}   
& \multicolumn{1}{c}{$84.20\textcolor{gray}{\pm 0.02}$}    
& \multicolumn{1}{c}{$\mathbf{88.42}\textcolor{gray}{\pm 0.16}$}     
& \multicolumn{1}{c}{$88.07\textcolor{gray}{\pm 0.05}$}    \\
& \multicolumn{1}{c}{40}             
& \multicolumn{1}{c}{-}     
& \multicolumn{1}{c}{$82.33\textcolor{gray}{\pm 1.60}$}     
& \multicolumn{1}{c}{$\mathbf{85.67}\textcolor{gray}{\pm 0.09}$}      
& \multicolumn{1}{c}{$85.57\textcolor{gray}{\pm 0.09}$} \\
\multicolumn{6}{c}{}\\
\multirow{2}{*}{Additive} 
& \multicolumn{1}{c}{20}
& \multicolumn{1}{c}{-}     
& \multicolumn{1}{c}{$72.55\textcolor{gray}{\pm 0.16}$}     
& \multicolumn{1}{c}{$\mathbf{87.87}\textcolor{gray}{\pm 0.33}$}      
& \multicolumn{1}{c}{$87.24\textcolor{gray}{\pm 0.16}$}    \\
& \multicolumn{1}{c}{40}             
& \multicolumn{1}{c}{-}     
& \multicolumn{1}{c}{$41.04\textcolor{gray}{\pm 1.13}$}    
& \multicolumn{1}{c}{$\mathbf{88.29}\textcolor{gray}{\pm 0.01}$}      
& \multicolumn{1}{c}{$83.89\textcolor{gray}{\pm 0.08}$} \\
\multicolumn{6}{c}{}\\
\multicolumn{6}{c}{Feature Attack}\\
\multicolumn{6}{c}{}\\
\multirow{2}{*}{Additive Noise} 
& \multicolumn{1}{c}{20}
& \multicolumn{1}{c}{-}   
& \multicolumn{1}{c}{$82.38\textcolor{gray}{\pm 0.13}$}    
& \multicolumn{1}{c}{$\mathbf{86.76}\textcolor{gray}{\pm 0.03}$}     
& \multicolumn{1}{c}{$86.63\textcolor{gray}{\pm 0.04}$}    \\
& \multicolumn{1}{c}{40}             
& \multicolumn{1}{c}{-}     
& \multicolumn{1}{c}{$78.54\textcolor{gray}{\pm 0.65}$}     
& \multicolumn{1}{c}{$\mathbf{82.27}\textcolor{gray}{\pm 0.06}$}      
& \multicolumn{1}{c}{$81.23\textcolor{gray}{\pm 0.03}$} \\
\multicolumn{6}{c}{}\\
\multirow{2}{*}{Impulse Noise} 
& \multicolumn{1}{c}{20}
& \multicolumn{1}{c}{$79.18\textcolor{gray}{\pm 6.47}$}     
& \multicolumn{1}{c}{$82.59\textcolor{gray}{\pm 0.60}$}     
& \multicolumn{1}{c}{$\textbf{86.91}\textcolor{gray}{\pm 0.36}$}      
& \multicolumn{1}{c}{$86.23\textcolor{gray}{\pm 0.03}$}    \\
& \multicolumn{1}{c}{40}             
& \multicolumn{1}{c}{-}     
& \multicolumn{1}{c}{$78.03\textcolor{gray}{\pm 0.73}$}    
& \multicolumn{1}{c}{$\mathbf{78.03}\textcolor{gray}{\pm 0.73}$}      
& \multicolumn{1}{c}{$81.41\textcolor{gray}{\pm 0.12}$} \\
\multicolumn{6}{c}{}\\
\multicolumn{6}{c}{Label Attack}\\
\multicolumn{6}{c}{}\\
\multirow{2}{*}{Backdoor} 
& \multicolumn{1}{c}{20}
& \multicolumn{1}{c}{$86.99\textcolor{gray}{\pm 0.02}$}   
& \multicolumn{1}{c}{$76.38\textcolor{gray}{\pm 0.13}$}    
& \multicolumn{1}{c}{$\mathbf{88.97}\textcolor{gray}{\pm 0.10}$}     
& \multicolumn{1}{c}{$88.26\textcolor{gray}{\pm 0.04}$}    \\
& \multicolumn{1}{c}{40}             
& \multicolumn{1}{c}{$73.01\textcolor{gray}{\pm 0.68}$}     
& \multicolumn{1}{c}{$60.85\textcolor{gray}{\pm 1.24}$}     
& \multicolumn{1}{c}{$\mathbf{84.69}\textcolor{gray}{\pm 0.31}$}      
& \multicolumn{1}{c}{$81.32\textcolor{gray}{\pm 0.16}$} \\\bottomrule 
\end{tabular}
\label{table:fmnist}
\end{table}

\begin{table}[htb!]
\centering
\caption{\textbf{Homogeneous Training:}1.12M parameter CNN trained on MNIST in regular i.i.d. setting. For all corruption types, test accuracy of \textsc{BGmD} is similar to that of \textsc{GmD} and surprisingly, in some cases even higher. As expected, \textsc{SGD} fails to make progress under corruption . \textsc{CmD} performs sub-optimally as corruption is increased.}
\begin{tabular}{llllll}
\toprule
& Corruption (\%) 
& \multicolumn{1}{c}{SGD} 
& \multicolumn{1}{c}{\textsc{CmD}} 
& \multicolumn{1}{c}{\textsc{BGmD}} 
& \multicolumn{1}{c}{\textsc{GmD}} \\ \midrule
Clean             
& \multicolumn{1}{c}{-}   
& \multicolumn{1}{c}{$\mathbf{99.27}\textcolor{gray}{\pm 0.01}$}    
& \multicolumn{1}{c}{$98.83\textcolor{gray}{\pm 0.02}$}    
& \multicolumn{1}{c}{$99.09\textcolor{gray}{\pm 0.05}$}     
& \multicolumn{1}{c}{$99.24\textcolor{gray}{\pm 0.02}$}    \\
\multicolumn{6}{c}{}\\
\multicolumn{6}{c}{Gradient Attack}\\
\multicolumn{6}{c}{}\\
\multirow{2}{*}{Bit Flip} 
& \multicolumn{1}{c}{20}
& \multicolumn{1}{c}{$9.51\textcolor{gray}{\pm 1.77}$}    
& \multicolumn{1}{c}{$98.79\textcolor{gray}{\pm 0.01}$}    
& \multicolumn{1}{c}{$\mathbf{99.06}\textcolor{gray}{\pm 0.02}$}     
& \multicolumn{1}{c}{$98.98\textcolor{gray}{\pm 0.01}$}    \\
& \multicolumn{1}{c}{40}             
& \multicolumn{1}{c}{$9.60\textcolor{gray}{\pm 2.04}$}     
& \multicolumn{1}{c}{$93.69\textcolor{gray}{\pm 0.09}$}     
& \multicolumn{1}{c}{$97.89\textcolor{gray}{\pm 0.05}$}      
& \multicolumn{1}{c}{$\mathbf{98.11}\textcolor{gray}{\pm 0.12}$} \\
\multicolumn{6}{c}{}\\
\multirow{2}{*}{Additive} 
& \multicolumn{1}{c}{20}
& \multicolumn{1}{c}{$9.68\textcolor{gray}{\pm 0.11}$}     
& \multicolumn{1}{c}{$94.26\textcolor{gray}{\pm 0.03}$}     
& \multicolumn{1}{c}{$98.61\textcolor{gray}{\pm 0.01}$}      
& \multicolumn{1}{c}{$\mathbf{98.69}\textcolor{gray}{\pm 0.01}$}    \\
& \multicolumn{1}{c}{40}             
& \multicolumn{1}{c}{$9.74\textcolor{gray}{\pm 0.12}$}     
& \multicolumn{1}{c}{$91.86\textcolor{gray}{\pm 0.03}$}    
& \multicolumn{1}{c}{$\mathbf{97.78}\textcolor{gray}{\pm 0.27}$}      
& \multicolumn{1}{c}{$92.78\textcolor{gray}{\pm 0.04}$} \\\bottomrule 
\end{tabular}
\label{table:mnist}
\end{table}

\begin{table}[htb!]
\centering
\caption{\textbf{Heterogeneous Training:} ResNet-18 trained on CIFAR-10 in a federated setting under two gradient attack models. We observe similar trends as Tables ~\ref{table:fmnist} and ~\ref{table:mnist} - \textsc{GmD} and \textsc{BGmD} both remain robust at both 20\% and 40\% corruption maintaining high accuracy. At high corruptions \textsc{BGmD} achieves similar or even better performance than \textsc{GmD} while taking ~3x less time.  
Interestingly, even in non-corrupt setting, \textsc{BGmD} outperforms \textsc{SGD}, perhaps due to the distributional shift introduced by non-i.i.d setting. Note that in i.i.d setting, our model achieves 94.8\% accuracy with \textsc{SGD}}
\begin{tabular}{llllll}
\toprule
& Corruption (\%) 
& \multicolumn{1}{c}{SGD} 
& \multicolumn{1}{c}{\textsc{CmD}} 
& \multicolumn{1}{c}{\textsc{BGmD}} 
& \multicolumn{1}{c}{\textsc{GmD}} \\ \midrule
Clean             
& \multicolumn{1}{c}{-}   
& \multicolumn{1}{c}{$82.29\textcolor{gray}{\pm 1.32}$}    
& \multicolumn{1}{c}{$85.50\textcolor{gray}{\pm 1.43}$}    
& \multicolumn{1}{c}{$84.82\textcolor{gray}{\pm 0.76}$}     
& \multicolumn{1}{c}{$\mathbf{85.65}\textcolor{gray}{\pm 0.48}$}    \\
\multicolumn{6}{c}{}\\
\multicolumn{6}{c}{Gradient Attack}\\
\multicolumn{6}{c}{}\\
\multirow{2}{*}{Bit Flip} 
& \multicolumn{1}{c}{20}
& \multicolumn{1}{c}{-}    
& \multicolumn{1}{c}{$80.87\textcolor{gray}{\pm 0.21}$}    
& \multicolumn{1}{c}{${84.56}\textcolor{gray}{\pm 0.06}$}     
& \multicolumn{1}{c}{$\mathbf{88.07}\textcolor{gray}{\pm 0.05}$}    \\
& \multicolumn{1}{c}{40}             
& \multicolumn{1}{c}{-}     
& \multicolumn{1}{c}{$77.41\textcolor{gray}{\pm 1.04}$}     
& \multicolumn{1}{c}{$\mathbf{82.66}\textcolor{gray}{\pm 0.31}$}      
& \multicolumn{1}{c}{$80.81\textcolor{gray}{\pm 0.01}$} \\
\multicolumn{6}{c}{}\\
\multirow{2}{*}{Additive} 
& \multicolumn{1}{c}{20}
& \multicolumn{1}{c}{$20.7\textcolor{gray}{\pm 1.56}$}     
& \multicolumn{1}{c}{$54.75\textcolor{gray}{\pm 0.38}$}     
& \multicolumn{1}{c}{$\mathbf{83.84}\textcolor{gray}{\pm 0.12}$}      
& \multicolumn{1}{c}{$82.40\textcolor{gray}{\pm 0.90}$}    \\
& \multicolumn{1}{c}{40}             
& \multicolumn{1}{c}{-}     
& \multicolumn{1}{c}{$23.35\textcolor{gray}{\pm 6.13}$}    
& \multicolumn{1}{c}{$\mathbf{82.79}\textcolor{gray}{\pm 0.68}$}      
& \multicolumn{1}{c}{$79.46\textcolor{gray}{\pm 0.24}$} \\\bottomrule 
\end{tabular}
\label{table:cifar}
\end{table}

\begin{figure}[t]
\subfloat[\footnotesize{No corruption}]{\includegraphics[width=0.25\columnwidth]{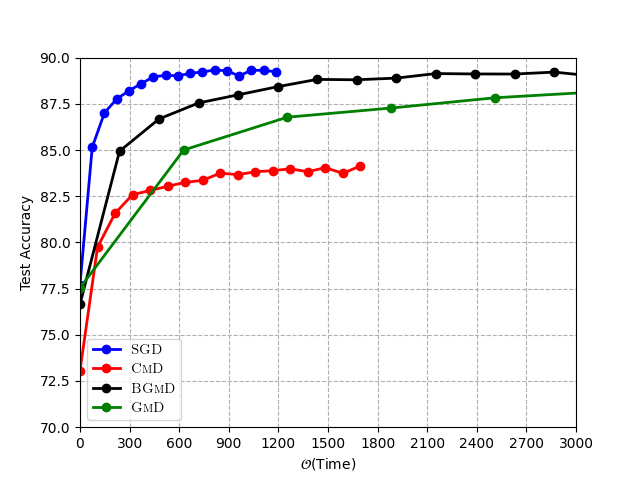}}
\subfloat[\footnotesize{10\% Corruption}]{\includegraphics[width=0.25\columnwidth]{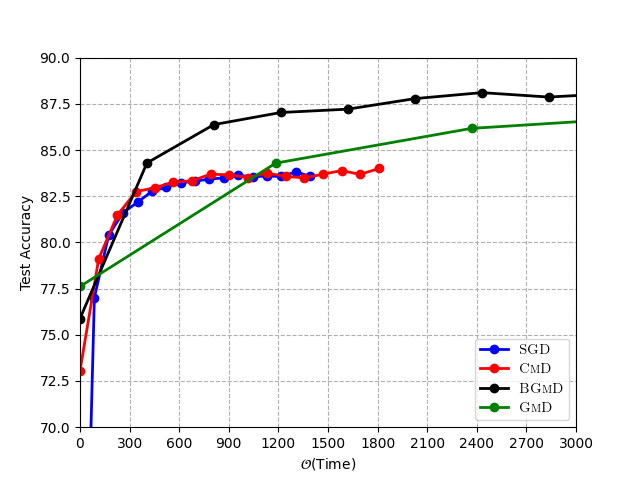}}
\subfloat[\footnotesize{20\% Corruption}]{\includegraphics[width=0.25\columnwidth]{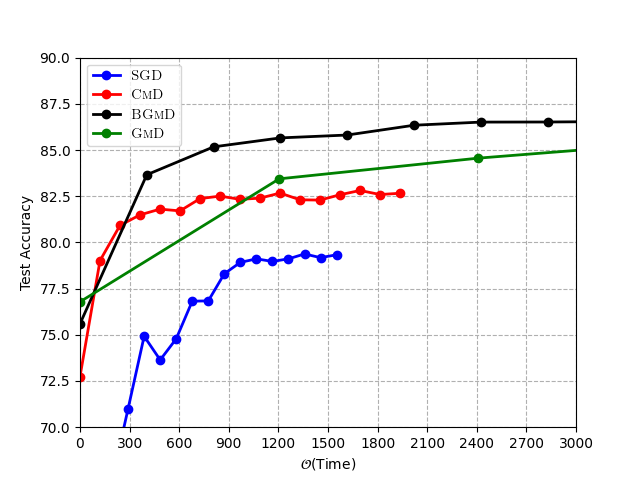}}
\subfloat[\footnotesize{40\% Corruption}]{\includegraphics[width=0.25\columnwidth]{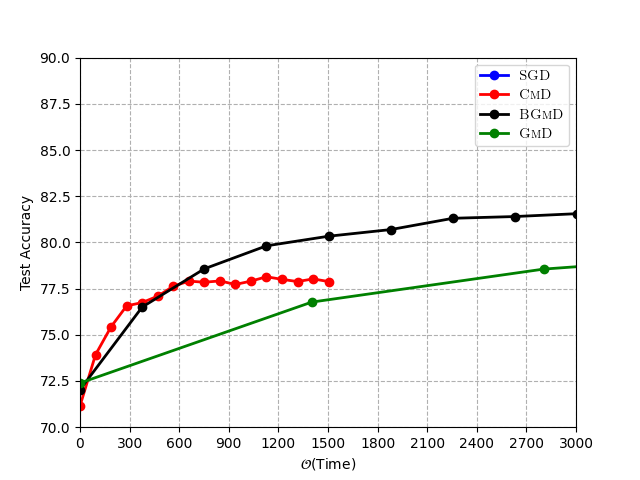}} 
\caption{\footnotesize{\textbf{Robustness to Feature Corruption}: We compare  test set performance of different schemes as a function of \textbf{wall clock time} for training Fashion MNIST in i.i.d setting. These experiments are performed in presence of increasing the fraction samples corrupted with \textbf{impulse noise}. At 40\% corruption \textsc{SGD} diverged and thus does not appear in the plot. Observe that \textsc{BGmD} is able to maintain high accuracy in even presence of strong corruption and is at least \textsc{GmD} while attaining at least 3x speedup over \textsc{GmD}.}}
\label{figure:fmnist.image_corr}
\end{figure}
\begin{figure}[t]
\subfloat[\footnotesize{No corruption}]{\includegraphics[width=0.25\columnwidth]{figs/neurips/fmnist/clean.png}}
\subfloat[\footnotesize{10\% Corruption}]{\includegraphics[width=0.25\columnwidth]{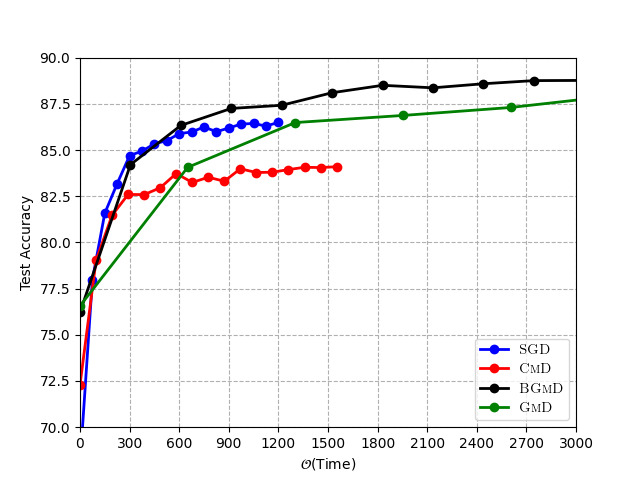}}
\subfloat[\footnotesize{20\% Corruption}]{\includegraphics[width=0.25\columnwidth]{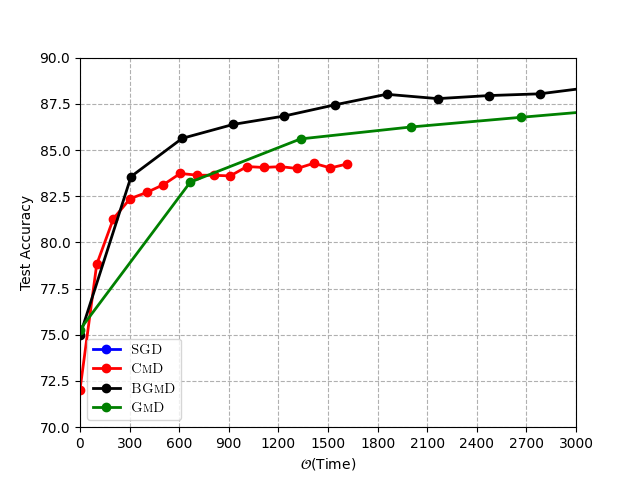}}
\subfloat[\footnotesize{40\% Corruption}]{\includegraphics[width=0.25\columnwidth]{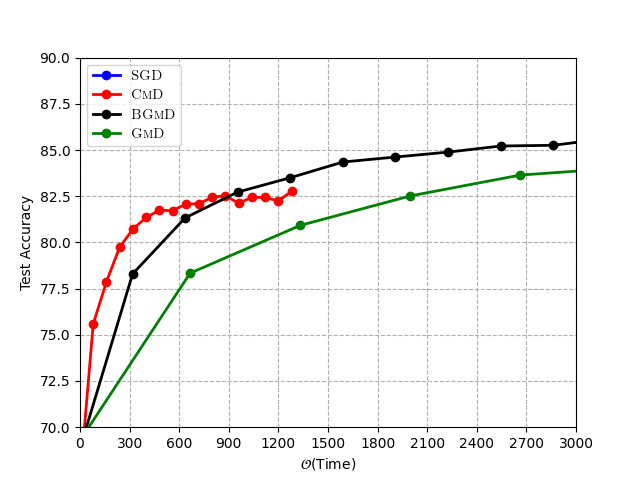}}  
\caption{\footnotesize{\textbf{Robustness to Gradient Corruption}: We train a CNN on Fashion MNIST in the i.i.d setting in presence of scaled bit flip corruption of stochastic gradients. Similar to Figure~\ref{figure:fmnist.image_corr}, \textsc{BGmD} achieves up to $\sim 3\times x$ speedup over \textsc{GmD} while being robust. For high corruption $\%$, \textsc{CmD} converges to sub-optimal solution.}}
\label{figure:fmnist.grad_corr}
\end{figure}
\begin{figure}[htb]
\subfloat[\footnotesize{No corruption}]{\includegraphics[width=0.25\columnwidth]{figs/neurips/fmnist/clean.png}}
\subfloat[\footnotesize{10\% Corruption}]{\includegraphics[width=0.25\columnwidth]{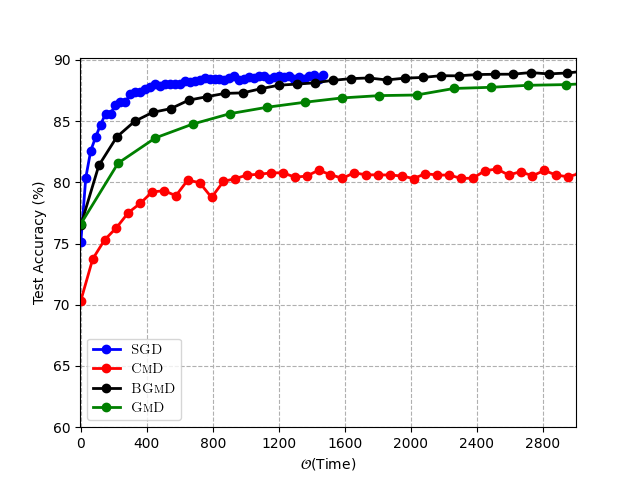}}
\subfloat[\footnotesize{20\% Corruption}]{\includegraphics[width=0.25\columnwidth]{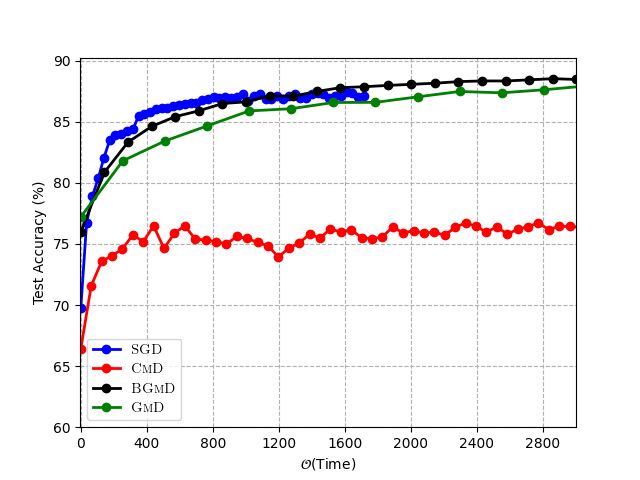}}
\subfloat[\footnotesize{40\% Corruption}]{\includegraphics[width=0.25\columnwidth]{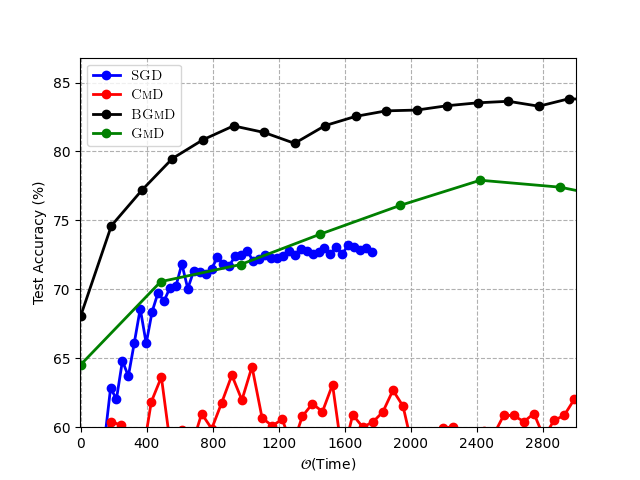}}  
\caption{\footnotesize{\textbf{Robustness to Label Corruption}: We train a CNN on Fashion-MNIST in the i.i.d setting in presence of \textbf{backdoor attack}. \textsc{BGmD} achieves up to large speedup over \textsc{GmD} while being robust even at very high corruptions where \textsc{SGD} converges very slowly due to backdoor attack and converges to a sub-optimal point at high 40\% corruption.\textsc{CmD} converges to sub-optimal solution.}}
\label{figure:fmnist.label_corr}
\end{figure}

\subsection{Corruption Simulation}
We consider all three possible sources of error: corruption in \textbf{features}, corruption in \textbf{labels}, and corruption in \textbf{communicated gradients}. All the experiments are repeated for 0\% (i.e. clean), 20\% and 40\% corruption levels, i.e, $\psi = 0, 0.2, 0.4$, respectively (see Definition~\ref{definition:byzantine}).  

\paragraph{\textbf{Feature Corruption:}} We consider corruption in the raw training data itself which can arise from different issues related to data collection. Particularly, adopting the corruptions introduced in \cite{hendrycks2018benchmarking}, we use two noise models to directly apply to the corrupt samples: \\
\textbf{(i) Additive}: Gaussian noise $\vz_i \sim \mathcal{N}(0 , 100)$ directly added to the image, and \\
\textbf{(ii) Impulse}: Salt and Pepper noise added by setting 90\% of the pixels to 0 or 1.

\begin{figure}[htb!]
\subfloat[Original Image]{\includegraphics[width=0.24\columnwidth]{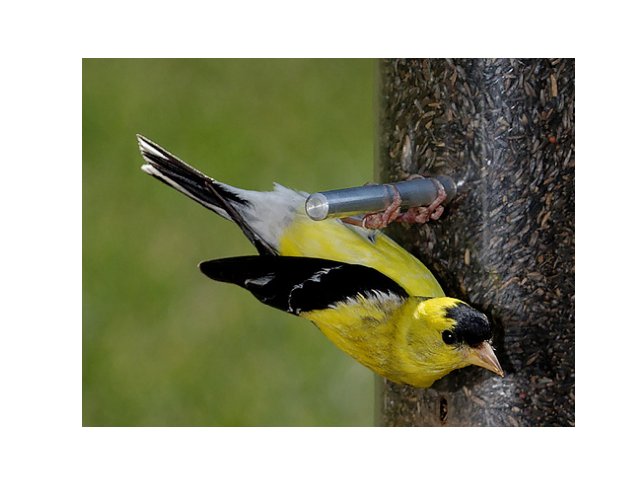}}
\subfloat[Gaussian Noise]{\includegraphics[width=0.24\columnwidth]{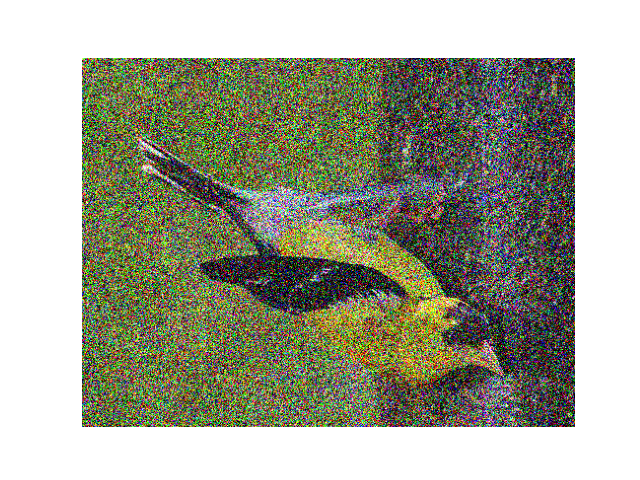}}
\subfloat[Salt \& Pepper Noise]{\includegraphics[width=0.24\columnwidth]{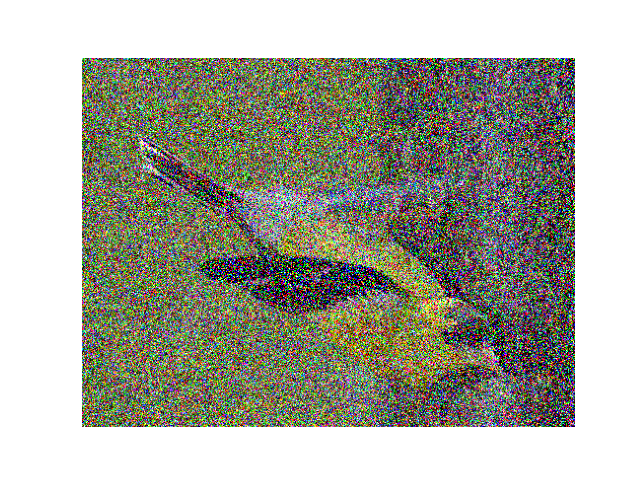}}
\subfloat[Gaussian Blur]{\includegraphics[width=0.24\columnwidth]{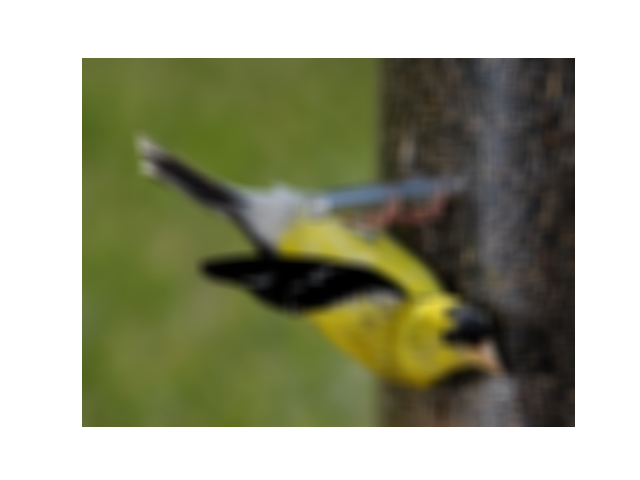}}
\caption{\small \textbf{Feature Corruption}: shows the effect of the perturbations added to image. }
\label{figure:corr_feat}
\end{figure}

\begin{figure}[htb!]
\subfloat[Clean Data]{\includegraphics[width=0.24\columnwidth]{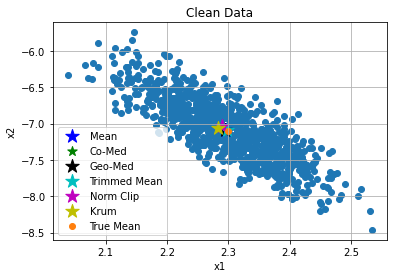}}
\subfloat[10\% corruption]{\includegraphics[width=0.24\columnwidth]{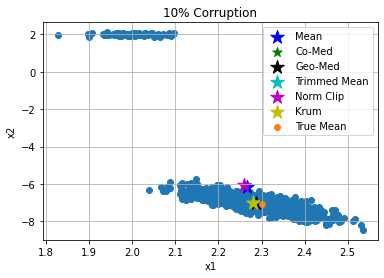}}
\subfloat[30\% corruption]{\includegraphics[width=0.24\columnwidth]{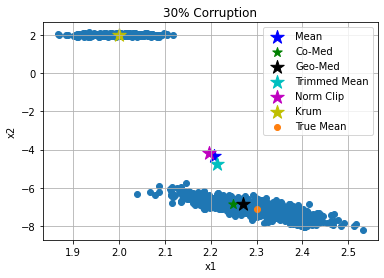}}
\subfloat[45\% corruption]{\includegraphics[width=0.24\columnwidth]{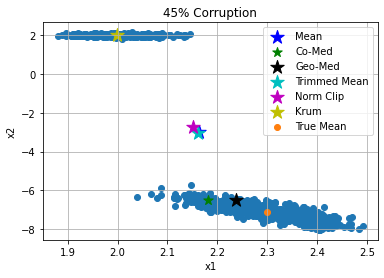}}
\caption{\small \textbf{Gradient Corruption}:This Toy example in 2 dimensions visually demonstrates the superior robustness properties of \textsc{Gm} for robust mean estimation (e.g. estimating the aggregated gradient) for increasing fraction of (None to 45\%) gradients corrupted with additive Gaussian noise.}
\label{figure:robust_gar_val}
\end{figure} 

\paragraph{\textbf{Gradient Corruption:}}
In distributed training over multiple machines the communicated gradients  can be noisy, e.g., due to hardware issues or simply because some nodes are adversarial and aim to maliciously disrupt the training. Using standard noise models for gradient corruption~\cite{fu1998penalized, xie2019zeno, bernstein2018signsgd} we directly corrupt the gradients in the following manner: \\ 
\textbf{(i) Additive}: adversary adds random Gaussian noise $\vz_i \sim \mathcal{N}(0, 100)$ to the true gradient, and\\
\textbf{(ii)Scaled Bit Flip}: corrupted gradients $\vg_t^{c}$ are  the scaled bit flipped version of the true gradient~\cite{bernstein2018signsgd} estimate; in particular, we use the following scale: $\vg_t^c = - 100 \vg_t$. 

\paragraph{\textbf{(c) Label Corruption:}} We simulate the important real world backdoor attack~\cite{shen2019learning, tolpegin2020data} where the goal of the adversary is to bias the classifier towards some adversary chosen class. To simulate this behavior: at each iteration we flip the labels of randomly chosen $\psi$ fraction of the samples to a \textbf{target} label (e.g. in Fashion-MNIST we use 8:bag as the backdoor label). 

\paragraph{\textbf{Dynamic Corruption Strategy.}}
In order to represent a wide array of real world adversarial scenarios we use the following dynamic corruption strategy: At each iteration, we randomly sample $\psi$ fraction of the samples malicious. Note that this implies none of the samples are trustworthy and thus makes the robust estimation problem truly unsupervised. Note that, this is a strictly stronger and realistic simulation of corruption compared to existing literature where a fixed set of samples are treated as malicious throughout the training. 

\subsection{Discussion}
\paragraph{\textbf{Discussion on the Test Performance}}
In Table~\ref{table:fmnist} we observe that without corruption both \textsc{BgmD} and \textsc{GmD} are able to achieve similar accuracy as the baseline (i.e., \textsc{SGD}). Conversely, \textsc{CmD} has a significant sub-optimality gap even in the clean setting. We observe the same trend under different corruption models over various levels of corruption. When corruption is high, \textsc{SGD} starts to diverge  after a few iterations. While \textsc{CmD} doesn't diverge, at higher level of corruptions its performance significantly degrades. On the other hand, both \textsc{GmD} and \textsc{BGmD} remain robust and maintain their test accuracy as expected from thier strong theoretical guarantees. Surprisingly, \textbf{\textsc{BGmD} not only maintains similar robustness as \textsc{GmD}, in several experiments it even outperforms} the performance of \textsc{GmD}. We hypothesize that this might be because of implicit regularization properties of \textsc{BGmD} in over-parameterized settings and leave it as a future problem. 

\paragraph{\textbf{Discussion on the Relative Speedup.}} 
The main goal of this work is to speed up \textsc{Gm} based robust optimization methods. To this end, we verify the efficiency of different methods via plotting the performance on test set as a function of the \textbf{wall clock time}. Figure~\ref{figure:fmnist.image_corr},~\ref{figure:fmnist.grad_corr} and~\ref{figure:fmnist.label_corr} suggest that under a variety of corruption settings \textsc{BgmD} is able to achieve significant speedup over \textsc{GmD} often by more than 3x while maintaining similar (sometime even better) test performance as \textsc{GmD}.

\paragraph{\textbf{Choice of $k$ and the Role of Memory}} 
In order to study the importance the memory mechanism we train a 4-layer CNN on MNIST~\cite{lecun1998gradient} dataset while retaining $0<\beta \leq 1$ fraction of  coordinates, i.e., $k = \beta d$. Figure~\ref{figure:choice_k} shows that the training runs that used memory (denoted by M) were able to achieve a similar accuracy while using significantly fewer coordinates. 
\setlength{\intextsep}{0pt}%
\setlength{\columnsep}{3pt}%
\begin{wrapfigure}{r}{0.36\textwidth}
\centering
\includegraphics[width=0.36\textwidth]{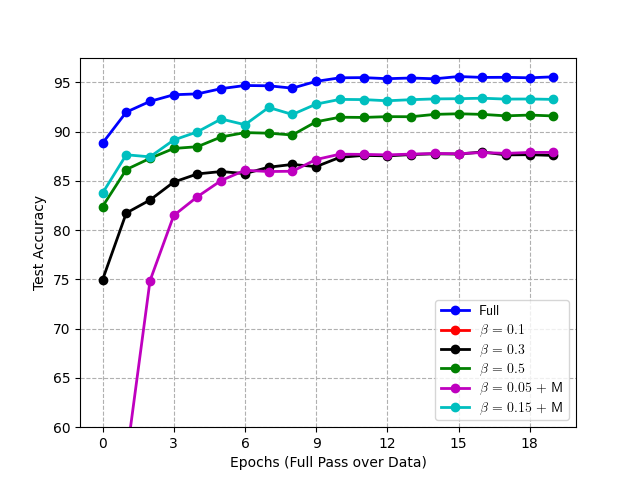}
\caption{\label{figure:choice_k}\footnotesize{\textsc{BGMD} with different values of $k$ and with / without memory on MNIST using CNN demonstrating the importance of memory augmentation}}
\vspace{-2mm}
\end{wrapfigure}

\paragraph{\textbf{Discussion on Over parameterization}}
Intuitively, as a consequence of over-parameterization, large scale deep learning models are likely to have sparse gradient estimates \cite{shi2019understanding}. In the context of robust optimization this would imply that performing robust gradient estimation in a low-dimensional subspace has little to no impact in the downstream optimization task especially in case of over-parameterized models.
Thus, by judiciously selecting a block of co-ordinates and performing robust gradient estimation only in this low-dimensional subspace is a practical approach to leverage highly robust estimators such as \textsc{Gm} even in high dimensional setting which was previously intractable. This means \textsc{BGmD} would be an attractive choice especially in overparameterized setting. In order to validate this, we perform a small experiment: we take a tiny 412 parameter CNN and a large 1.16M parameter CNN and train both on simple MNIST dataset. In addition to looking at the train and test performance we also compute relative residual error $1 - \xi = \|\mG_t - \gC_k(\mG_t)\|_F^2/\|\mG_t\|_F^2$  Lemma~\ref{lemma:norm_sampling} at each gradient step (iteration). We plot the relative residual, train and test performance for different values of $k = \beta d$ where $0 < \beta \leq 1$ in Figure~\ref{figure:supp.role.k.lenet}, ~\ref{figure:supp.role.k.tiny}. \footnote{Just to clarify, both these experiments are performed in clean setting}

\begin{figure}[htb]
\subfloat[\footnotesize{Relative Residual}]{\includegraphics[width=0.33\columnwidth]{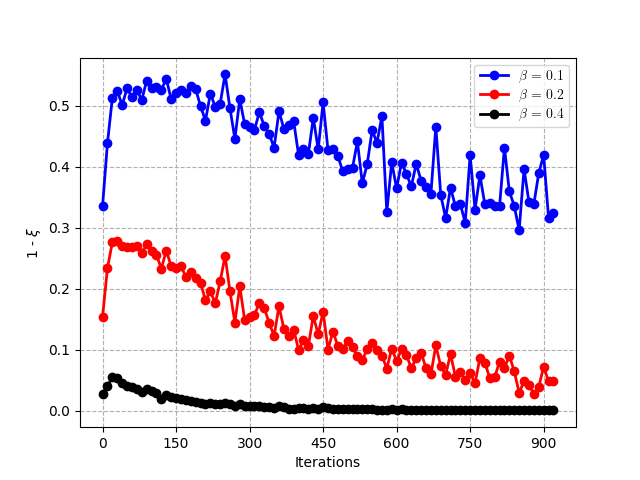}}
\subfloat[\footnotesize{Test Accuracy}]{\includegraphics[width=0.33\columnwidth]{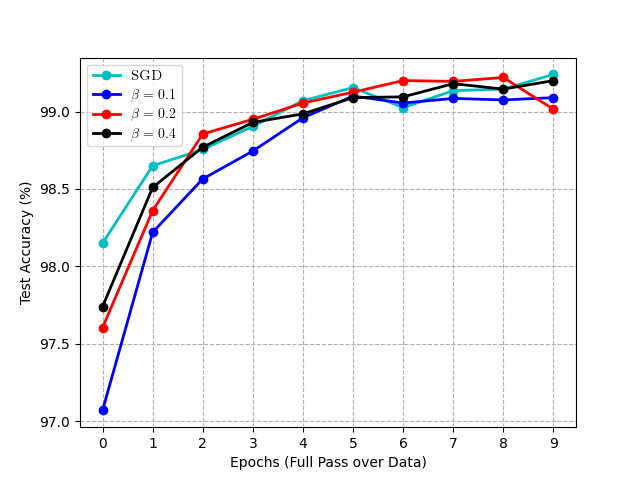}}
\subfloat[\footnotesize{Train Loss}]{\includegraphics[width=0.33\columnwidth]{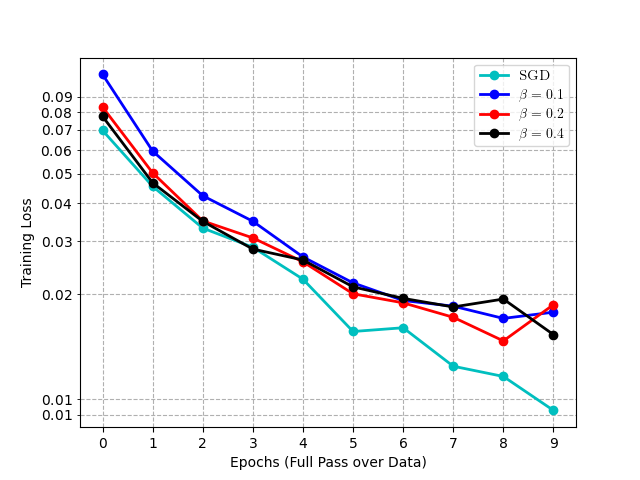}}
\caption{\footnotesize{\textbf{Large Model}: We train an over-parameterized ($d$ = 1.16 M) parameter CNN on MNIST~\cite{lecun1998gradient} dataset. (a) shows relative residual error as a function of iterations (gradient steps) (b) and (c) show the corresponding test and training performance. The residual error starts to approaches zero as training progresses even for very small values of $k$ explaining negligible impact on training and test performance.}} 
\label{figure:supp.role.k.lenet}
\end{figure}

\begin{figure}[htb]
\subfloat[\footnotesize{Relative Residual}]{\includegraphics[width=0.33\columnwidth]{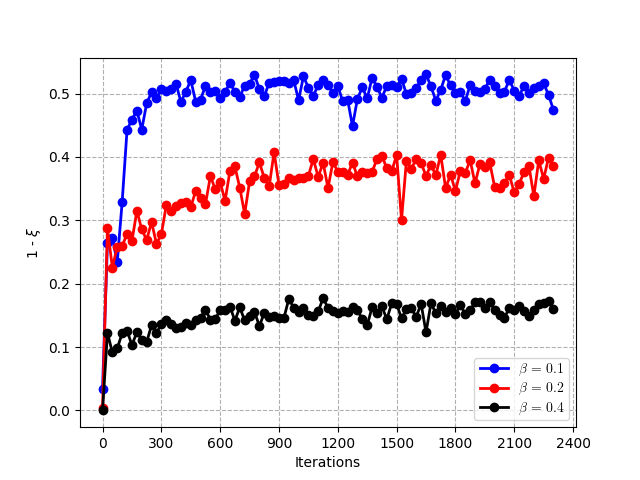}}
\subfloat[\footnotesize{Test Accuracy}]{\includegraphics[width=0.33\columnwidth]{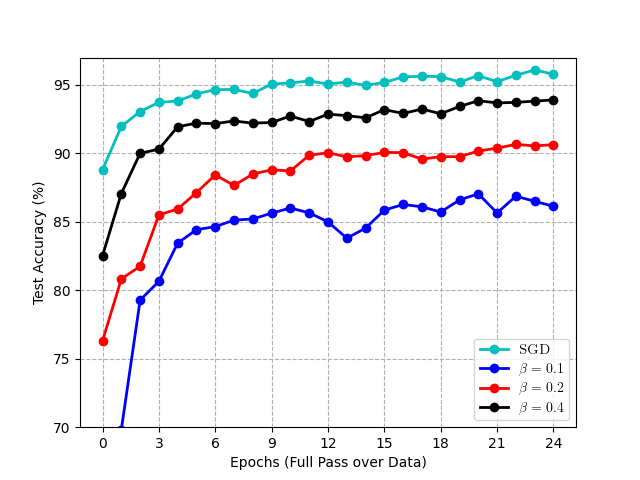}}
\subfloat[\footnotesize{Train Loss}]{\includegraphics[width=0.33\columnwidth]{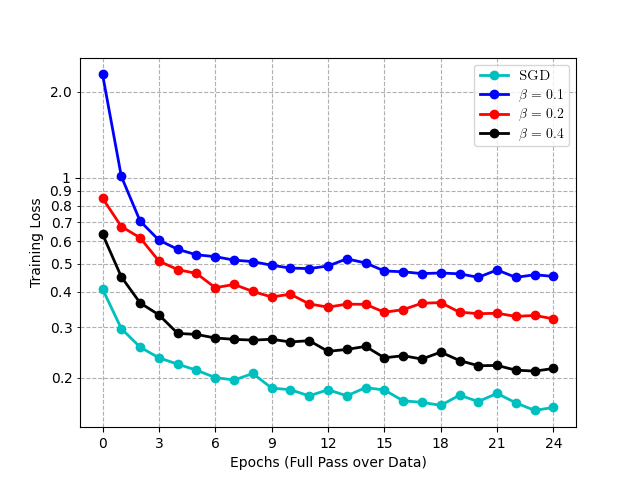}}
\caption{\footnotesize{\textbf{Tiny Model}: Here we instead train a super tiny $d=412$ parameter CNN on MNIST. (a),(b),(c) plots relative residual, train and test performance respectively. We can see that even in this setting $(1 - \xi) > k/d = \beta$ as we remarked in Lemma~\ref{lemma:norm_sampling}. However, as training progresses the gradients become closer to uniform explaining the optimality gap of \textsc{BGmD} with small values of $k$.}}
\label{figure:supp.role.k.tiny}
\end{figure}

It is clear from these experiments that in over-parameterized settings $\xi \to 0$ as training progresses potentially due to the inherent sparsity of overparameterized models. This implies our block selection procedure is often near lossless and thus \textsc{BGmD} even with aggressively small $k$ can achieve near optimal performance. And since it is a few orders of magnitude faster than \textsc{GmD} - it is an attractive robust optimization algorithm in high dimensions. However, in the tiny CNN setting, as training progresses the gradients tend to become more uniform ~ implying higher loss due to block co-ordinate selection. This implies, in these settings, \textsc{BGmD} needs to use a larger $k$ in order to retain the optimal performance. However, note that in these small scale settings, \textsc{GmD} is not expensive and thus for robust optimization , a larger value of $k$ or at the extreme $k=d$ i.e. running \textsc{GmD} can be practical. 

\paragraph{\textbf{Summary of Results}}
The above empirical observations highlight the following insights: \\
\textbf{(a)} For challenging corruption levels/models, \textsc{Gm} based methods are indeed superior while standard \textsc{SGD} or \textsc{CmD} can be significantly inaccurate, \\
\textbf{(b)} In all of our reported experiments, \textsc{BGmD} was run with $k$ set to 10\% of the number of parameters. Despite using such a small ratio \textsc{BGmD} retains the strong performance of \textsc{GmD}.\\
\textbf{(c)} By judiciously choosing $k$, \textsc{BGmD} is more efficient than \textsc{GmD}, and \\
\textbf{(d)} memory augmentation is vital for \textsc{BGmD} to attain a high accuracy while employing relatively small values of $k$.

\section{Conclusion}
\label{section:conclusion}
We proposed \textsc{BGmD}, a method for robust, large-scale, and high dimensional optimization that achieves  the optimal statistical breakdown point while delivering  significant savings in the computational costs per iteration compared to existing \textsc{Gm}-based strategies. \textsc{BGmD} employs  greedy coordinate selection and memory augmentation which allows to aggressively select very few coordinates while attaining strong convergence properties comparable to  \textsc{Gm}-SGD under standard non-convex settings. Extensive deep learning experiments demonstrated the efficacy of \textsc{BGmD}. 

\section{Limitation} Currently,   \textsc{BGmD} samples the coordinates for computing geometric median using just the norm of each coordinate in the gradient matrix. In large deep networks, the gradients should have more structure that cannot be captured just by the norm of coordinates. We 
 leave further investigation into leveraging gradient structure for more effective coordinate selection for future work. 
 
\section{Broader Impact}
While the current work is more theoretical and algorithmic in nature, we believe that robust learning -- the main problem addressed in the work -- is a key requirement for deep learning systems to ensure that a few malicious data points do not completely de-rail the system and produce offensive/garbage output.

\clearpage
\appendix
\begin{center}
    \textbf{\Large Missing Proofs of \textsc{BgmD}}\vspace{5mm}
\end{center}
\section{Proof of Lemma \ref{lemma:norm_sampling} (Sparse Approximation)}
\label{proof:norm_sampling}
\newtheorem*{lemma1*}{Lemma \ref{lemma:norm_sampling}}
\begin{lemma1*}
Algorithm \ref{algorithm:co_sparse_approx} yields a contraction approximation, i.e., $\E_{\gC_k} \left[\|\gC_k(\mG)-\mG\|_F^2 |\vx\right] \leq (1-\xi)\|\mG\|_F^2 , \; \frac{k}{d} \leq \xi \leq 1,$ 
where $\gC_k(\mG)_{i, j\in \omega_k} = \mG_{i, j},$ and $\gC_k(\mG)_{i, j\notin \omega_k} =0$.
\end{lemma1*}
\begin{proof}
Suppose, $\Omega_k = \{\omega \subseteq \{1,2, ..., d\}: |\omega|=k\}$ is the set of all possible subsets of cardinality $k$ i.e. $|\Omega_k|=\binom{d}{k}$. Also let the embeddings produced by random co-ordinate sampling and active norm sampling (Algorithm \ref{algorithm:co_sparse_approx}) are denoted by $\gC^r_k(\cdot)$ and $\gC^{n}_k(\cdot)$ respectively and let the $i-$th row $\mG_t[i,:] = \vg_t^i$. Then, we can bound the reconstruction error in expectation $\forall \mG_t \in \sR^{b \times d}$ as:
\begin{align*}
\E_{\gC^{n}_k} \left[\|\gC^{n}_k(\mG_t)-\mG_t\|_F^2 |\mG\right]
&\leq \E_{\gC^{n}_k} \left[\|\gC^{r}_k(\mG_t)-\mG_t\|^2 |\mG_t\right]\\
&=\sum_{i=1}^n \E_{\gC^{n}_k} \left[\|\gC^{r}_k(\vg_t^i)-\vg_t^i\|^2 |\vg_t^i\right]\\
&= \sum_{i=1}^n \frac{1}{|\Omega_k|}\sum_{\omega \in \Omega_k}\sum_{i=1}^{d}\vx_i^2 \mathbb{I}\{i\notin \omega\}\\
&=\sum_{i=1}^n(1-\frac{k}{d})\|\vg_t^i\|^2\\
& = (1-\frac{k}{d}) \|\mG_t\|^2_F
\end{align*} 
This concludes the proof.
\end{proof}

\section{Computational Complexity of \textsc{BGmD} iterates}
\begin{proposition}
\label{proposition:bgmd_cost}\textup{\textbf{(Computational Complexity).}}
Given an $\epsilon$- approximate \textsc{Gm} oracle , each gradient aggregation of \textsc{BGmD} with block size $k$ incurs a computational cost of: $\gO(\frac{k}{\epsilon^2} + bd)$.
\end{proposition}

\begin{proof}
Let us first recall, one gradient aggregation operation: given $b$ stochastic gradients $\vg_i \; \forall i \in [b]$, the goal of the aggregation step is to compute $\Tilde{\vg}$ for subsequent iterations of the form $\vx_{t+1} = \vx_t - \eta \Tilde{g}$. Further, $\mG_t \in \sR^{b \times d}$ denote the gradient jacobian where $i$ th row $\mG_t[i, :]$ corresponds to $\vg_i$. Also assume that we have at our disposal an oracle $\textsc{Gm}(\cdot)$ that can compute an  $\epsilon$-approximate \textsc{Gm} $\vg_i \in \sR^d \; \forall i \in [b]$ using $\gO(\frac{d}{\epsilon^2})$ compute \cite{cohen2016geometric}. Recall that \textsc{BGmD} (Algorithm \ref{algorithm:full}) is composed of the following main steps: \\*[0.3 cm] 
\textbullet \; \textbf{Memory Augmentation}: At each gradient aggregation step, \textit{\textsc{BgmD} needs to add back the stored memory $\hat{\vm}_t$ compensating for accumulated residual error incurred in previous iterations to $\mG_t$ such that $\mG_t[i, :] = \gamma \mG_t[i, :] + \hat{\vm}_t$. Note that, this is a row wise linear (addition) operation implying $\gO(bd)$ associated cost}.\\*[0.2 cm] 
\textbullet \; \textbf{Active Norm Sampling}: At each gradient aggregation step: \textit{\textsc{BgmD} selects $k$ of the $d$ coordinates i.e. $k$ columns of $\mG_t$ using Algorithm \ref{algorithm:co_sparse_approx}. This requires computing the $\ell_2$ norm distribution along the $d$ columns, followed by sampling $k$ of them proportional to the norm distribution. The computational complexity of computing  Active Norm Sampling is $\gO(bd)$} \cite{wang2017provably}.\\*[0.2 cm] 
\textbullet \; \textbf{Compute $\mM_{t+1}$}: Further, memory needs to be updated to $\mM_{t+1}$ for future iterates implying another row wise linear operation incurring $\gO(bd)$ compute.\\*[0.2 cm] 
\textbullet \; \textbf{Low Rank \textsc{Gm}}: \textit{Note that \textsc{BGmD} needs to run $\textsc{Gm}(\cdot)$ over $\sR^k$ implying a cost of $\gO(\frac{k}{\epsilon^2})$}. \\*[0.2 cm] 
Putting it together, the total cost of computing gradient aggregation per iteration using Algorithm \ref{algorithm:full} is then $\gO(\frac{k}{\epsilon^2} + bd)$. This concludes the proof. 
\end{proof}

\section{Proof of Lemma \ref{proposition:compute}}
\newtheorem*{proposition*}{Lemma \ref{proposition:compute}}
\begin{proposition*}
\textbf{(Choice of k)}
Let $k \leq \gO(\frac{1}{F} - b\epsilon^2)\cdot d$. Then, given an $\epsilon$- approximate \textsc{Gm} oracle,  Algorithm \ref{algorithm:full} achieves a factor $F$ speedup over \textsc{Gm-SGD} for aggregating $b$ samples.
\end{proposition*}
\begin{proof}
First, note that: for one step of gradient aggregation \textsc{Gm-SGD} makes one call to \textsc{Gm}($\cdot$) oracle implying a $\gO(\frac{d}{\epsilon^2})$ computational cost per gradient aggregation.\\*[0.2 cm] 
Now, let us assume, $k = \beta d$ where $0 < \beta \leq 1$ denotes the fraction of total gradient dimensions retained by \textsc{BGmD}. Then using Proposition \ref{proposition:bgmd_cost} we can find a bound on $\beta$ such that gradient aggregation step of \textsc{BGmD} has a linear speedup over that of \textsc{Gm-SGD} by a linear factor $F$. 
\begin{equation}
    \begin{aligned}
    & \gO(\frac{d}{\epsilon^2}) \geq F \cdot \gO(\frac{k}{\epsilon^2} + bd)\\
    \Rightarrow & \gO(\frac{1}{\epsilon^2}) \geq \gO(\frac{F \beta}{\epsilon^2}) + \gO(Fb)\\
    \Rightarrow & \gO(F \beta) \leq \gO(1) - \gO(Fb\epsilon^2)\\
    \Rightarrow & \gO(\beta) \leq \gO(\frac{1}{F}) - \gO(b\epsilon^2)
    \end{aligned}
\end{equation}
This concludes the proof.
\end{proof}
\section{Detailed Statements of the convergence Theorems}
We here state the full version of the main convergence theorems 
\newtheorem*{theorem1*}{Theorem \ref{theorem:nonconvex-simple}}
\begin{theorem1*}
[\textbf{Smooth Non-convex}]
Consider the general case where the functions $f_i$ correspond to non-corrupt samples $i \in \sG$ i.e. $f=\frac{1}{\sG}\sum_{i\in \sG} f_i(\mathbf{x})$ are \textbf{non-convex} and \textbf{smooth} (Assumption \ref{assumption:smooth}). Define, $R_0 := f(\vx_0)-f(\vx^\ast)$ where $\vx^\ast$ is the true optima and $\vx_0$ is the initial parameters. Run Algorithm \ref{algorithm:full} with compression factor $\xi$ (Lemma \ref{lemma:norm_sampling}), learning rate $\gamma =  1/2L$ and $\epsilon-$approximate \textup{\textsc{Gm}($\cdot$)} oracle in presence of $\alpha-$corruption (Definition \ref{definition:byzantine}) for $T$ iterations. Then it holds that: 
\begin{equation}
\label{eq:thm:nonconvex-app}
\begin{aligned}
\frac{1}{T} \sum_{t=0}^{T-1}\E\|\nabla f(\vx_t)\|^2 &\leq 
\frac{8R_0}{\gamma T} + 8L\gamma\sigma^2 + \frac{48L^2 \gamma^2\sigma^2(1-\xi^2)}{\xi^2} \\&+ \frac{2304 \sigma^2}{(1-\alpha)^2}\left[1+\frac{4(1-\xi^2)}{\xi^2}\right] + \frac{48\epsilon^2}{\gamma^2|\sG|^2(1-\alpha)^2}
\\&=\gO\Bigg( \frac{LR_0}{T}+
\frac{\sigma^2\xi^{-2}}{(1-\alpha)^2}+\frac{L^2\epsilon^2}{|\sG|^2(1-\alpha)^2}
\Bigg)
\end{aligned}
\end{equation}
\end{theorem1*}
\newtheorem*{theorem2*}{Theorem \ref{theorem:plc-simple}}
\begin{theorem2*}
[\textbf{Non-convex under PLC}]
Assume in addition to \textbf{non-convex} and \textbf{smooth} (Assumption \ref{assumption:smooth}) the functions $f_i$ correspond to non-corrupt samples
also satisfy the \textbf{Polyak-\L{}ojasiewicz Condition} (Assumption \ref{assumption:plc}) with parameter $\mu$. 
After $T$ iterations Algorithm \ref{algorithm:full} with compression factor $\xi$ (Lemma \ref{lemma:norm_sampling}), learning rate $\gamma =  1/4L$ and $\epsilon-$approximate \textup{\textsc{Gm}($\cdot$)} oracle in presence of $\alpha-$corruption (Definition \ref{definition:byzantine}) satisfies:
\begin{equation}
\begin{aligned}
\E\|\hat{\vx}_T-\vx^\ast\|^2 &\leq
 \frac{16(f(\vx_0)-f^\ast)}{\mu^2\gamma}\left[1-\frac{\mu\gamma}{2}\right]^T+\frac{16L\gamma\sigma^2}{\mu^2}+ \left[\frac{80\gamma^2\sigma^2L^2(1-\xi^2)}{\mu^2\xi^2}\right]\\
&+\frac{3072 \sigma^2}{\mu^2(1-\alpha)^2}\left[1+\frac{4(1-\xi^2)}{\xi^2}\right]+\frac{64\epsilon^2}{\mu^2\gamma^2 |\sG|^2(1-\alpha)^2}\\&= \gO\Bigg( \frac{LR_0}{\mu^2}\left[1-\frac{\mu}{8L}\right]^T+\frac{\sigma^2\xi^{-2}}{\mu^2(1-\alpha)^2}+\frac{L^2\epsilon^2}{\mu^2 |\sG|^2(1-\alpha)^2}\Bigg)
 \end{aligned}
 \end{equation}
for a global optimal solution  $\vx^\ast \in \gX^\ast$. 
\\ 
Here, $\hat{\vx}_T := \frac{1}{W_T}\sum_{t=0}^{T-1}w_t\vx_t$ with weights $w_t := (1-\frac{\mu}{8L})^{-(t+1)}$, $W_T := \sum_{t=0}^{T-1}w_t$.
\end{theorem2*}
\section{Useful Facts}
\begin{fact}
\label{fact:1}
$
\E\|\sum_{i\in \gA} \va_i\|^2 \leq |\gA|\sum_{i\in \gA}\E\| \va_i\|^2
$ \\
This can be seen as a consequence of the Jensen's inequality.
\end{fact}
\begin{fact}[\textbf{Young's Inequality}]
\label{fact:2} 
For any $\beta>0$,
\begin{equation}
    \E[\langle\va,\vb\rangle]\leq \frac{\beta}{2}\E\| \va\|^2+ \frac{1}{2\beta}\E\| \vb\|^2
\end{equation}
This can be seen as a special case of the weighted AM-GM inequality.
\end{fact}
\begin{fact}[\textbf{Lemma 2 in \cite{stich2019unified}}]
\label{fact:3}
Let $\{a_t\}$ and $\{b_t\}$ be to non-negative sequences such that
\begin{equation}
    a_{t+1} \leq (1-r\gamma)a_t-s\gamma b_t+c,
\end{equation}
where $r,\gamma,s,c>0$ and $r\gamma<1$. Let $w_t = (1-r\gamma)^{-(t+1)}$ and $W_T = \sum_{t=0}^{T-1}w_t$. Then the following holds: 
\begin{equation}
    \frac{s}{W_T}\sum_{t=0}^{T-1}b_tw_t+ra_{T} \leq \frac{a_0}{\gamma}(1-r\gamma)^T+\frac{c}{\gamma}.
\end{equation}
\end{fact}
\begin{fact}
\label{fact:4}
Let $f$ be a function that satisfies PLC (Assumption \ref{assumption:plc}) with parameter $\mu$. Then, $f$ satisfies the quadratic growth condition \textup{\cite{karimi2016linear, zhang2020new}}:
\begin{equation}
    f(\vx)-f^\ast \geq \frac{\mu}{2}\|\vx-\vx_p\|^2
\end{equation}
where $\vx_p$ is the projection of $\vx$ onto
the solution set $\gX^\ast$ . 
\end{fact}
\section{Proof of Lemma \ref{lemma:bound-e}}
\begin{proof}
The result is due to Lemma 5 in \cite{basu2019qsparse} which is inspired by Lemma 3 in \cite{karimireddy2019error}. The proof relies on using the fact that the proposed norm sampling operator $\gC_k$, as shown in Lemma \ref{lemma:norm_sampling} is contractive. Hence, we can use this property to derive a recursive bound on the norm of the memory. Using this result as well as the bounded stochastic gradient assumption results in a geometric sum that can be bounded by the RHS of \eqref{equation:bound-e}.
\end{proof}
\section{Proof of Lemma \ref{lemma:bound-z-simplified}}
\begin{proof}
Since $\Tilde{\vg}_t$ is the $\epsilon$-accurate \textsc{Gm} of $\{\Delta_t^i\}_i$, $\vz_t$ can be thought of as the $\epsilon$-accurate \textsc{Gm} of $\{\Delta_t^i-\Delta_t\}_i$. Thus, by the classical robustness property of \textsc{Gm} (Theorem. 2.2 in \cite{lopuhaa1991breakdown}; see also \cite{minsker2015geometric,cohen2016geometric,chen2017distributed,li2019rsa,wu2020federated} for similar adaptations) we have
\begin{equation}
\label{equation:bound-z0-simplified}
    \E\|\vz_t\|^2\leq \frac{8|\sG|}{(|\sG|-|\sB|)^2}\sum_{i\in \sG}\E\|\Delta_t^i-\Delta_t\|^2+\frac{2\epsilon^2}{(|\sG|-|\sB|)^2}.
\end{equation}
Next, we bound $\E\|\Delta_t^i-\Delta_t\|^2$ using the properties of memory mechanism stated in Lemma \ref{lemma:bound-e} with $H=1$ as:
\begin{equation}
\label{equation:bound-z1-simplified}
\begin{aligned}
  \E\|\Delta_t^i-\Delta_t\|^2&= 2\E\|\Delta_t^i\|^2+2\E\|\Delta_t\|^2 \\&\leq 2\E\|\Delta_t^i\|^2+\frac{2}{|\sG|}\sum_{i\in \sG}\E\|\Delta_t^i\|^2\\&\leq
  4\max_{i\in \sG}\E\|\Delta_t^i\|^2,
  \end{aligned}
\end{equation}
where we used Fact \ref{fact:1} twice. Next, we establish a bound on the norm of the communicated messages as follows. Add add and  subtract $\vp_t^i$ and use the update rule of the memory to obtain
\begin{equation}
\label{equation:bound-z2-simplified}
\begin{aligned}
  \E\|\Delta_t^i\|^2&= \E\|\Delta_t^i+\vp_t^i-\vp_t^i\|^2 \\&= \E\|\vp_t^i- \hat{\vm}_{t+1}\|^2\\&=\E\|\gamma\vg_t^i+\hat{\vm}_{t}-\hat{\vm}_{t+1}\|^2\\&\leq 3\E\|\gamma\vg_t^i\|^2+3\E\|\hat{\vm}_{t}\|^2+3\E\|\hat{\vm}_{t+1}\|^2,
  \\&\leq 3\gamma^2\sigma^2+3\E\|\hat{\vm}_{t}\|^2+3\E\|\hat{\vm}_{t+1}\|^2, 
  \end{aligned}
\end{equation}
by definition of $\vp_t^i$ and Fact \ref{fact:1}. Notice that using Lemma \ref{lemma:bound-e} with $H=1$ we can uniformly bound the last two terms on the RHS of \eqref{equation:bound-z2-simplified}:
\begin{equation}\label{equation:bound-z22-simplified}
\begin{aligned}
  3\E\|\hat{\vm}_{t}\|^2+3\E\|\hat{\vm}_{t+1}\|^2 \leq \frac{24(1-\xi^2)\gamma^2\sigma^2}{\xi^2}.
  \end{aligned}
\end{equation}
Therefore, we conclude
\begin{equation}\label{equation:bound-z5-simplified}
\begin{aligned}
\E\|\Delta_t^i\|^2&\leq3\gamma^2\sigma^2+\frac{24(1-\xi^2)\gamma^2\sigma^2}{\xi^2}.
    \end{aligned}
\end{equation}
Therefore, by \eqref{equation:bound-z1-simplified} and \eqref{equation:bound-z5-simplified}
\begin{equation}
\label{equation:bound-z6-simplified}
\begin{aligned}
  \E\|\Delta_t^i-\Delta_t\|^2& \leq 12\gamma^2\sigma^2\left[1+\frac{4(1-\xi^2)}{\xi^2}\right],
  \end{aligned}
\end{equation}
and the proof is complete by combining this last result with \eqref{equation:bound-z0-simplified}.
\end{proof}
\section{Proof of Lemma \ref{lemma:bound-virtual-simplified}}
\begin{proof}
We derive a recursive relation for the difference $\vx_{t+1} - \tilde{\vx}_{t+1} $.
It follows that
\begin{equation}
\label{equation:bound-virtual:3-simplified}
\begin{aligned}
    \vx_{t+1} &= \vx_t-\Tilde{\vg}_t \\&= \vx_t - \Delta_t-\vz_t.
    \end{aligned}
\end{equation}
On the other hand,
\begin{equation}
\label{equation:bound-virtual:4-simplified}
\begin{aligned}
    \tilde{\vx}_{t+1} &= \tilde{\vx}_{t} -\gamma \vg_t  -\gamma \vz_t \\&=\tilde{\vx}_{t} -\gamma \vg_t -\vz_t \\&=\tilde{\vx}_{t}-\gamma \vg_t-\vz_t,
    \end{aligned}
\end{equation}
Collectively, \eqref{equation:bound-virtual:3-simplified} and \eqref{equation:bound-virtual:4-simplified} imply
\begin{equation}
\label{equation:bound-virtual:5-simplified}
\begin{aligned}
    \vx_{t+1} - \tilde{\vx}_{t+1} &=(\vx_{t} - \tilde{\vx}_{t})+(\gamma\vg_t-\Delta_t).
    \end{aligned}
\end{equation}
Since $\vy_{0} =\tilde{\vy}_{0}$ using induction yields
\begin{equation}
\label{equation:bound-virtual:6-simplified}
\begin{aligned}
    \vx_{t+1} - \tilde{\vx}_{t+1} &=\sum_{j=0}^t(\gamma\vg_t-\Delta_j)\\&=\sum_{j=0}^t(\vm_{j+1}-\vm_j)\\&=\vm_{t+1}-\vm_0=\vm_{t+1},
    \end{aligned}
\end{equation}
where we used the fact that $\vm_0=\mathbf{0}$.
\end{proof}

\section{Proof of Lemma \ref{lemma:recursive_bound}} 
\begin{proof}
Recall that $\vg_t = \frac{1}{|\sG|}\sum_{i\in\sG}\nabla f_i(\vx_t,z_t^i)$, i.e., the average of stochastic gradients ove uncorrupted samples at time $t$, and $\E[\vg_t] = \bar{\vg_t}$. By the definition of $L$-smoothness (see Assumption \ref{assumption:smooth}), we have
\begin{equation}
\label{equation:simplified:thm:1}
\begin{aligned}
    f(\tilde{\vx}_{t+1}) &\leq f(\tilde{\vx}_{t})+ \langle \nabla f(\tilde{\vx}_{t}), \tilde{\vx}_{t+1}-\tilde{\vx}_{t}\rangle+\frac{L}{2}\|\tilde{\vx}_{t+1}-\tilde{\vx}_{t}\|^2\\
    &= f(\tilde{\vx}_{t})+ \langle \nabla f(\tilde{\vx}_{t}), -\gamma \vg_t- \vz_t\rangle+\frac{L}{2}\|\gamma \vg_t+ \vz_t\|^2\\
    &\leq f(\tilde{\vx}_{t}) -\gamma \langle \nabla f(\tilde{\vx}_{t}),  \vg_t \rangle -\langle\nabla f(\tilde{\vx}_{t}), \vz_t\rangle+L\gamma^2\| \vg_t\|^2 + L\| \vz_t\|^2.
    \end{aligned}
\end{equation}
Let $\E_t$ denote expectation with respect to sources of randomness in computation of stochastic gradients at time $t$. Then,
\begin{equation}
\label{equation:simplified:thm:2}
\begin{aligned}
\E_t[f(\tilde{\vx}_{t+1})] &\leq f(\tilde{\vx}_{t})-\gamma \langle \nabla f(\tilde{\vx}_{t}),  \bar{\vg}_t \rangle - \E_t[\langle\nabla f(\tilde{\vx}_{t}),\vz_t\rangle]+L\gamma^2\E_t\| \vg_t\|^2 + L\E_t\|\vz_t\|^2.
 \end{aligned}
 \end{equation}
 The first inner-product in \eqref{equation:simplified:thm:2} can be bounded according to
 \begin{equation}
 \label{equaton:simplified:thm:3}
\begin{aligned}
2\langle \nabla f(\tilde{\vx}_{t}),  \bar{\vg}_t \rangle &= \|\nabla f(\tilde{\vx}_{t})\|^2 + \|\bar{\vg}_t\|^2-\|\nabla f(\tilde{\vx}_{t})-  \bar{\vg}_t\|^2\\
 &\geq  \|\nabla f(\tilde{\vx}_{t})\|^2 -\|\nabla f(\tilde{\vx}_{t})-  \bar{\vg}_t\|^2\\
&= \|\nabla f(\tilde{\vx}_{t})\|^2 -\| \frac{1}{|\sG|}\sum_{i\in\sG}\nabla f_i(\tilde{\vx}_{t})-  \frac{1}{|\sG|} \sum_{i\in\sG} \nabla f_i(\vx_t)\|^2\\
&\geq \|\nabla f(\tilde{\vx}_{t})\|^2 -\frac{1}{|\sG|}\sum_{i\in\sG}\| \nabla f_i(\tilde{\vx}_{t})-  \nabla f_i(\vx_t)\|^2\\
&\geq \|\nabla f(\tilde{\vx}_{t})\|^2 -L^2\|\tilde{\vx}_{t}-  \vx_t\|^2\\
&= \|\nabla f(\tilde{\vx}_{t})\|^2 +L^2\|\tilde{\vx}_{t}-  \vx_t\|^2-2L^2\|\tilde{\vx}_{t}-  \vx_t\|^2\\
&\geq \|\nabla f(\tilde{\vx}_{t})\|^2 +\|\nabla f(\tilde{\vx}_{t})-  \nabla f(\vx_t)\|^2-2L^2\|\tilde{\vx}_{t}-\vx_t\|^2\\
&= \|\nabla f(\tilde{\vx}_{t})\|^2 +\|\nabla f(\tilde{\vx}_{t})-  \nabla f(\vx_t)\|^2-2L^2\|\tilde{\vx}_{t}-\vx_t\|^2\\
&\geq  \frac{1}{2}\|  \nabla f(\vx_t)\|^2-2L^2\|\tilde{\vx}_{t}-\vx_t\|^2,
 \end{aligned}
 \end{equation}
 where we employed  Fact \ref{fact:1} and $L$-smoothness of each function several times. Therefore, 
  \begin{equation}\label{equation:simplified:thm:1inner}
\begin{aligned}
-\gamma\langle \nabla f(\tilde{\vx}_{t}),  \bar{\vg}_t \rangle \leq -\frac{\gamma}{4} \|  \nabla f(\vx_t)\|^2+\gamma L^2\|\tilde{\vx}_{t}-\vx_t\|^2.
 \end{aligned}
 \end{equation}
 We now bound the second inner-product. To this end,
 \begin{equation}
 \label{equation:simplified:thm:2inner}
 \begin{aligned}
     -\E_t[\langle\nabla f(\tilde{\vx}_{t}),\vz_t\rangle] &= -\E_t[\langle \nabla f(\vx_t),\vz_t\rangle]+\E_t[\langle\nabla f(\vx_t)-\nabla f(\tilde{\vx}_{t}),\vz_t\rangle]\\
     &\leq \rho\gamma\|  \nabla f(\vx_t)\|^2+\frac{1}{2\rho\gamma}\E_t\|\vz_t\|^2+\E_t[\langle\nabla f(\vx_t)-\nabla f(\tilde{\vx}_{t}),\vz_t\rangle]\\
     &\leq\frac{\rho\gamma}{2} \|  \nabla f(\vx_t)\|^2+\frac{1}{2\rho\gamma}\E_t\|\vz_t\|^2+\frac{1}{2\gamma}\E_t\|\vz_t\|^2 + \frac{\gamma}{2}\|\nabla f(\vx_t)-\nabla f(\tilde{\vx}_{t})\|^2\\
     &\leq\frac{\rho\gamma}{2} \|  \nabla f(\vx_t)\|^2+\left(\frac{1}{2\rho\gamma}+\frac{1}{2\gamma}\right)\E_t\|\vz_t\|^2+\frac{\gamma L^2}{2}\|\vx_t-\tilde{\vx}_{t}\|^2,
     \end{aligned}
 \end{equation}
 where we used Fact \ref{fact:2} twice and to obtain the last inequality we employed the smoothness assumption. Here, $0<\rho<0.5$ is a parameter whose value will be determined later.
 
Application of \eqref{equation:simplified:thm:1inner} and \eqref{equation:simplified:thm:2inner} in  \eqref{equation:simplified:thm:2} yields
\begin{equation}
\label{equation:simplified:thm:4}
\begin{aligned}
\E_t[f(\tilde{\vx}_{t+1})] &\leq f(\tilde{\vx}_{t})-\left(\frac{1}{2}-\rho\right)\frac{\gamma}{2} \|  \nabla f(\vx_t)\|^2+\frac{3\gamma L^2}{2}\|\tilde{\vx}_{t}-\vx_t\|^2\\
&\qquad\qquad+L\gamma^2\E_t\|\vg_t\|^2+ \left(L+\frac{1}{2\rho\gamma}+\frac{1}{2\gamma}\right)\E_t\|\vz_t\|^2.
    \end{aligned}
\end{equation}
\end{proof}

\section{Proof of Theorem \ref{theorem:nonconvex-simple}}
\begin{proof} 
Rearranging \eqref{equation:simplified:thm:4} and taking expectation with respect to the entire sources of randomness, i.e. the proposed coordinated sparse approximation and the randomness in computation of stochastic gradient in iterations $0,\dots,t-1$, using the bounded SFO assumption yields 
\begin{equation}
\label{equation:simplified:thm:5}
\begin{aligned}
\left(\frac{1}{2}-\rho\right)\frac{\gamma}{2}\E\|\nabla f(\vx_t)\|^2 \leq &\E[f(\tilde{\vx}_{t})]-\E[f(\tilde{\vx}_{t+1})]+\frac{3\gamma L^2}{2}\E\|\tilde{\vx}_{t}-\vx_t\|^2\\
& + L\gamma^2\sigma^2+ \left(L+\frac{1}{2\rho\gamma}+\frac{1}{2\gamma}\right)\E\|\vz_t\|^2.
    \end{aligned}
\end{equation}
Evidently, we need to bound two quantities in \eqref{equation:simplified:thm:5}. Lemma \ref{lemma:bound-z-simplified} establishes a bound on $\E\|\vz_t\|^2$ while by Lemma \ref{lemma:bound-virtual-simplified}
\begin{equation}\label{equation:simplified:lemma7_replacement}
    \begin{aligned}
    \E\|\tilde{\vx}_{t}-\vx_t\|^2 = \E\|\vm_t\|^2 &\leq \frac{1}{|\sG|} \sum_{i\in \sG} \E\|\hat{\vm}_t\|^2\leq \frac{4(1-\xi^2)\gamma^2\sigma^2}{\xi^2}
    \end{aligned}
\end{equation}
using Lemma \ref{lemma:bound-e} with $H=1$ we get :
\begin{equation}
\label{equation:simplified:thm:6}
\begin{aligned}
\left(\frac{1}{2}-\rho\right)\frac{\gamma}{2}\E\|\nabla f(\vx_t)\|^2 \leq &\E[f(\tilde{\vx}_{t})]-\E[f(\tilde{\vx}_{t+1})]+L\gamma^2\sigma^2\\
&+\frac{6\gamma L^2(1-\xi^2)\gamma^2\sigma^2}{\xi^2}
\\
&+ \left(L+\frac{1}{2\rho\gamma}+\frac{1}{2\gamma}\right) \frac{96\gamma^2\sigma^2}{(1-\alpha)^2}\left[1+\frac{4(1-\xi^2)}{\xi^2}\right]\\
&+ \left(L+\frac{1}{2\rho\gamma}+\frac{1}{2\gamma}\right)\frac{2\epsilon^2}{|\sG|^2(1-\alpha)^2}.
    \end{aligned}
\end{equation}
Finally, letting $\rho = 1/4$ and $\gamma = \frac{1}{2L}$ we obtain the stated result by averaging \eqref{equation:simplified:thm:6} over time and noting $f^\ast\leq\E[f(\tilde{\vx}_{T})]$ and $f(\tilde{\vx}_{0}) = f(\vx_0)$. 
\end{proof}

\section{Proof of Theorem \ref{theorem:plc-simple}}
Recall \eqref{equation:simplified:thm:5} in the proof of Theorem \ref{theorem:nonconvex-simple}. Let $a_{t} :=\E[f(\tilde{\vx}_{t})] -f^\ast \geq 0$. Using Assumption \ref{assumption:plc} to bound the gradient terms in \eqref{equation:simplified:thm:5} yields
\begin{equation}\label{equation:simplified:thm:plc1}
\begin{aligned}
a_{t+1} \leq & a_t - \left(\frac{1}{2}-\rho\right)\gamma\mu  \E[f(\vx_t)-f^\ast] + \frac{3\gamma L^2}{2} \E\|\tilde{\vx}_{t}-\vx_t\|^2\\
& + L\gamma^2\sigma^2+ \left(L+\frac{1}{2\rho\gamma}+\frac{1}{2\gamma}\right)\E\|\vz_t\|^2.
    \end{aligned}
\end{equation}
The above result is not sufficient to complete the proof since $a_{t}$ is with respect to the virtual sequence. This difficulty, which is addressed for the first time in this paper,  has been the main challenge in the proof of convergence of error compensated schemes with biased gradient compression under PLC. To deal with this burden, we revisit \eqref{equation:simplified:thm:2} to establish an alternative bound. Specifically, we to bound the first inner-product in \eqref{equation:simplified:thm:2} we instead establish
\begin{equation}
\label{equation:simplified:thm:plc2}
\begin{aligned}
-\gamma \langle \nabla f(\tilde{\vx}_{t}),  \bar{\vg}_t \rangle &= -\gamma \langle \nabla f(\tilde{\vx}_{t}),  \nabla f(\tilde{\vx}_{t}) \rangle+\gamma \langle \nabla f(\tilde{\vx}_{t}),  \nabla f(\tilde{\vx}_{t})-\bar{\vg}_t \rangle\\
&\leq -\gamma\| \nabla f(\tilde{\vx}_{t})\|^2+\frac{\gamma\rho_1}{2}\| \nabla f(\tilde{\vx}_{t})\|^2+\frac{\gamma}{2\rho}\|\nabla f(\tilde{\vx}_{t})-\bar{\vg}_t\|^2\\
&= -\gamma(1-\frac{\rho_1}{2})\| \nabla f(\tilde{\vx}_{t})\|^2+\frac{\gamma}{2\rho_1}\|\nabla f(\tilde{\vx}_{t})-\bar{\vg}_t\|^2
    \end{aligned}
\end{equation}
where we used Fact \ref{fact:2} with parameter $\rho_1>0$ which will be determined later. Recall that by smoothness
\begin{equation}
\label{equation:simplified:thm:plc3}
\begin{aligned}
\|\nabla f(\tilde{\vx}_{t})-\bar{\vg}_t\|^2\leq L^2 \|\tilde{\vx}_{t}- \vx_t\|^2.
 \end{aligned}
 \end{equation}
 We now derive a bound for the second inner-product in \eqref{equation:simplified:thm:2}. To do so, an application of Fact \ref{fact:2} yields 
 \begin{equation}
 \label{equation:simplified:thm:plc4}
\begin{aligned}
-\E_t[\langle\nabla f(\tilde{\vx}_{t}),\vz_t\rangle]\leq \frac{\rho_2\gamma}{2}\|\nabla f(\tilde{\vx}_{t})\|^2+\frac{1}{2\gamma\rho_2}\E_t\|\vz_t\|^2.
 \end{aligned}
 \end{equation}
 Therefore, in light of these new bounds we obtain
\begin{equation}
\label{equation:simplified:thm:plc5}
\begin{aligned}
\E_t[f(\tilde{\vx}_{t+1})] &\leq f(\tilde{\vx}_{t})-\gamma\left(1-\frac{\rho_1+\rho_2}{2}\right)\|\nabla f(\tilde{\vx}_{t})\|^2+L\gamma^2\E_t\| \vg_t\|^2\\
&+ \left(L+\frac{1}{2\gamma\rho_2}\right)\E_t\|\vz_t\|^2+\frac{\gamma L^2}{2\rho_1} \|\tilde{\vx}_{t}- \vx_t\|^2.
 \end{aligned}
 \end{equation}
Subtracting $f^\ast$ and taking expectation with respect to the entire sources of randomness yields
\begin{equation}
\label{equation:simplified:thm:plc6}
\begin{aligned}
\E[f(\tilde{\vx}_{t+1})]-f^\ast &\leq \E[f(\tilde{\vx}_{t})]-f^\ast-\gamma\left(1-\frac{\rho_1+\rho_2}{2}\right)\E\|\nabla f(\tilde{\vx}_{t})\|^2+L\gamma^2\sigma^2\\
&+ \left(L+\frac{1}{2\gamma\rho_2}\right)\E\|\vz_t\|^2+\frac{\gamma L^2}{2\rho_1} \E\|\tilde{\vx}_{t}- \vx_t\|^2.
 \end{aligned}
 \end{equation}
The gradient term $\E\|\nabla f(\tilde{\vx}_{t})\|^2$ in \eqref{equation:simplified:thm:plc6} can be related to $a_{t}$ using the PL condition. Therefore
\begin{equation}
\label{equation:simplified:thm:plc7}
\begin{aligned}
a_{t+1} \leq &\left[1-2\mu\gamma\left(1-\frac{\rho_1+\rho_2}{2}\right)\right]a_t+L\gamma^2\sigma^2\\
&+ \left(L+\frac{1}{2\gamma\rho_2}\right)\E\|\vz_t\|^2+\frac{\gamma L^2}{2\rho_1} \E\|\tilde{\vx}_{t}- \vx_t\|^2.
 \end{aligned}
 \end{equation}
 Multiply \eqref{equation:simplified:thm:plc1} and \eqref{equation:simplified:thm:plc7} by $1/2$ and add the two to obtain
 \begin{equation}\label{equation:simplified:thm:plc8}
\begin{aligned}
a_{t+1} \leq &\left[1-\mu\gamma\left(1-\frac{\rho_1+\rho_2}{2}\right)\right]a_t+L\gamma^2\sigma^2-\left(\frac{1}{2}-\rho\right)\frac{\gamma\mu}{2}  \E[f(\vx_t)-f^\ast]\\
&\qquad+ \left(L+\frac{1}{4\gamma\rho}+\frac{1}{4\gamma\rho_2}+\frac{1}{4\gamma}\right)\E\|\vz_t\|^2\\
&\qquad+\left(3+\frac{1}{\rho_1}\right)\frac{\gamma L^2}{4} \E\|\tilde{\vx}_{t}- \vx_t\|^2.
 \end{aligned}
 \end{equation}
 Define
 \begin{equation}\label{equation:simplified:thm:plc9}
     b_t :=   \E[f(\vx_t)-f^\ast].
 \end{equation}
 Let $\rho=1/4$, $\rho_1=1/2$, $\rho_2=1/2$, and $\gamma = 1/4L$. 
 Using Lemma \ref{lemma:bound-z-simplified} and \eqref{equation:simplified:lemma7_replacement}, \eqref{equation:simplified:thm:plc8} simplifies to
\begin{equation}\label{equation:simplified:thm:plc10}
\begin{aligned}
a_{t+1} &\leq \left[1-\frac{\mu\gamma}{2}\right]a_t-\frac{\gamma\mu}{8} b_t+L\gamma^2\sigma^2+\frac{5 L^2}{4} \gamma^3 \sigma^2\left[\frac{4(1-\xi^2)}{\xi^2}\right]\\
&\qquad+\frac{192\gamma  \sigma^2}{(1-\alpha)^2}\left[1+\frac{4(1-\xi^2)}{\xi^2}\right]+\frac{4\epsilon^2}{\gamma |\sG|^2(1-\alpha)^2}.
 \end{aligned}
 \end{equation}
 Thus we can apply Fact \ref{fact:3} to obtain
 \begin{equation}
 \label{equation:simplified:thm:plc11}
\begin{aligned}
\frac{1}{W_T}\sum_{t=0}^{T-1}b_tw_t &\leq \frac{8(f(\vx_0)-f^\ast)}{\mu\gamma}\left[1-\frac{\mu\gamma}{2}\right]^T+\frac{8L\gamma\sigma^2}{\mu}+\frac{10 L^2}{\mu} \gamma^2 \sigma^2\left[\frac{4(1-\xi^2)}{\xi^2}\right]\\
&+\frac{1536  \sigma^2}{\mu(1-\alpha)^2}\left[1+\frac{4(1-\xi^2)}{\xi^2}\right]+\frac{32\epsilon^2}{\mu\gamma^2 |\sG|^2(1-\alpha)^2},
 \end{aligned}
 \end{equation}
 where we used $E[f(\tilde{\vx}_0^i)] = f(\vx_0)$.
 Finally, to obtain the stated result we invoke Fact \ref{fact:4} and the assumption that the solution set (i.e., the set of all stationary points) is convex. Specifically, by the quadratic  growth condition and convexity of the Euclidean norm 
\begin{equation}
\label{equation:simplified:thm:plc12}
\begin{aligned}
  \sum_{t=0}^{T-1}\frac{w_t}{W_T}\E[f(\vx_t)-f^\ast] &\geq   \sum_{t=0}^{T-1}\frac{w_t}{W_T} \E\|\vx_t-\gP_{\gX^\ast}(\vx_t)\|^2\\
&\geq \frac{\mu}{2} \sum_{t=0}^{T-1}\frac{w_t}{W_T} \E\|\vx_t-\gP_{\gX^\ast}(\vx_t)\|^2\\
&\geq \frac{\mu}{2}  \E\|\sum_{t=0}^{T-1}\frac{w_t}{W_T}\vx_t-\sum_{t=0}^{T-1}\frac{w_t}{W_T}\gP_{\gX^\ast}(\vx_t)\|^2\\
&:=\frac{\mu}{2}  \E\|\hat{\vx}_T-\vx^{\ast}\|^2,
 \end{aligned}
 \end{equation}
 where $\gP_{\gX^\ast}(\vx_t)$ is the projection of $\vx_t$ onto the solution set $\gX^\ast$ and $\vx^{\ast}:=\sum_{t=0}^{T-1}\frac{w_t}{W_T}\gP_{\gX^\ast}(\vx_t) \in \gX^\ast$ by the assumption that the solution set is convex.
\clearpage
\bibliographystyle{alpha}
\bibliography{bibs/robust_estimation, bibs/robust_sgd, bibs/co_decent, bibs/compression,
bibs/het, bibs/batch_sgd, bibs/sparse_approx, bibs/misc_exp, bibs/attack_models, bibs/sketching}
\end{document}